\documentclass[12pt]{article}
\usepackage{amsmath,amsthm,amsfonts,amssymb}
\usepackage{mathrsfs,bm}
\usepackage{graphicx,amscd,epsfig,psfrag,verbatim}

\usepackage{setspace}
\usepackage{mathpazo,bm}

\usepackage{hyperref}
\hypersetup{
    colorlinks = true,
    citecolor = {blue},
    linkcolor = {blue},
}
\usepackage{natbib}
\usepackage{booktabs}
\usepackage{makecell}

\usepackage{algorithm}
\usepackage{algpseudocode}

\usepackage[a4paper, margin=0.7in]{geometry}

\usepackage{xcolor}

\newtheorem{theorem}{Theorem}

\newtheorem{lemma}[theorem]{Lemma}
\newtheorem{corollary}[theorem]{Corollary}
\newtheorem{proposition}[theorem]{Proposition}

\newcommand{\numberthis}{\addtocounter{equation}{1}\tag{\theequation}}

\usepackage{amsfonts,amssymb,mathrsfs,bm,bbm,dsfont}

\renewcommand{\d}{\text{\rm\,d}}

\def\R{\mathbb{R}}

\def\N{\mathbb{N}}
\def\E{\mathbb{E}}
\def\P{\mathbb{P}}

\def\eps{\varepsilon}
\def\del{\delta}

\def\L{\Lambda}
\def\la{\lambda}
\def\t{\theta}
\def\d{\mathrm{d}}
\def\a{\alpha}
\def\b{\beta}
\def\S{\Sigma}
\def\calL{\mathcal{L}}

\def\s{\sigma}

\def\tr{\mathrm{Tr}}

\def\ind{{\mathds{1}}}
\def\cC{\mathcal{C}}

\def\sgn{\mathrm{sgn}}

\def\cG{\mathcal{G}}

\def\X{\mathbb{X}}
\def\calC{\mathcal{C}}

\def\ind{\mathbbm{1}}

\def\low{\mathrm{low}}

\def\slow{\mathrm{Spec}_{\mathrm{low}}}
\def\spec{\mathrm{Spec}}
\def\tflow{t_{\mathrm{flow}}}
\def\Ns{N_{\mathrm{step}}}
\def\deg{\mathrm{deg}}

\def\md{\mathfrak{m}}
\def\snd{\mathfrak{d}}
\def\cP{\mathcal{P}}

\newcommand{\bzero}{\mathbf{0}}
\newcommand{\bone}{\mathbf{1}}

\renewcommand{\l}[0]{\left }
\renewcommand{\r}[0]{\right}
\newcommand{\op}{\mathrm{op}}
\newcommand{\RE}{\mathrm{RE}}

\usepackage{xcolor}

\newcommand{\bk}{\color{black}}

\let\hat\widehat
\let\tilde\widetilde

\title{\bf Learning under Latent Group Sparsity \\ via Diffusion on Networks}
\author{Subhroshekhar Ghosh \\
    Department of Mathematics \\
    National University of Singapore \\
    10 Lower Kent Ridge Road \\
    Singapore 119076 \\
    E-mail: \texttt{subhrowork@gmail.com} \\
    \url{https://subhro-ghosh.github.io}
    \and
    Soumendu Sundar Mukherjee \\
    Statistics and Mathematics Unit \\
    Indian Statistical Institute \\
    203 Barrackpore Trunk Road \\
    Kolkata, India 700108 \\
    E-mail: \texttt{ssmukherjee@isical.ac.in} \\
    \url{https://soumendu041.gitlab.io}
}

\date{}

\begin{document}
\maketitle

\begin{abstract}
    Group or cluster structure on explanatory variables in machine learning problems is a very general phenomenon, which has attracted broad interest from practitioners and theoreticians alike. In this work we contribute an approach to sparse learning under such group structure, that does not require prior information on the group identities. Our paradigm is motivated by the Laplacian geometry of an underlying network with a related community structure, and proceeds by directly incorporating this into a penalty that is effectively computed via a heat-flow-based local network dynamics. The proposed penalty interpolates between the lasso and the group lasso penalties, the runtime of the heat-flow dynamics being the interpolating parameter. As such it can automatically default to lasso when the group structure reflected in the Laplacian is weak. In fact, we demonstrate a data-driven procedure to construct such a network based on the available data. Notably, we dispense with computationally intensive pre-processing involving clustering of variables, spectral or otherwise. Our technique is underpinned by rigorous theorems that guarantee its effective performance and provide bounds on its sample complexity. In particular, in a wide range of settings, it provably suffices to run the diffusion for time that is only logarithmic in the problem dimensions. We explore in detail the interfaces of our approach with key statistical physics models in network science, such as the Gaussian Free Field and the Stochastic Block Model. 
Our work raises the possibility of applying similar diffusion-based techniques to classical learning tasks, exploiting the interplay between geometric, dynamical and stochastic structures underlying the data.
\end{abstract}


\section{Introduction}\label{sec:intro}

\subsection{Background}
The understanding and analysis of data with complex structure is a leitmotif of modern science and technology.  The spectacular growth in the capacity and computational means to process gigantic volumes of data has motivated the development of novel analytical paradigms in recent years.  A common theme that characterises many of these approaches is that they seek to incorporate the growing complexity that is inherent in such massive data sets \citep{national2013frontiers,marx2013big}.  

The intrinsic structure in data can manifest itself in various forms.  These range from the almost ubiquitous scenario of sparsity in an appropriate basis,  such as in compressive sensing and low-rank estimation problems \citep{achlioptas2007fast,grasedyck2013literature,chang2000adaptive,foucart2013invitation,berthet2013optimal},  to algebraic constraints imposed by physical considerations,  such as  symmetries under rigid motions that are inherent in problems of cryo-electron microscopy \citep{cheng2015primer,singer2018mathematics,hadani2011representation,perry2019sample,fan2020likelihood,ghosh2021multi}. In yet other instances,  constraints may be stochastic in nature,  pertaining to the statistical dependency structures that  characterise  the model \citep{mezard2009information,ros2019complex,ghosh2020fractal,lauritzen2019maximum,li2016fast,lavancier2015determinantal,ghosh2020gaussian,bardenet2021determinantal}.

\subsection{Statistical problems with a latent group structure}
A significant structural feature that arises in  many application scenarios is the clustering, or grouping, of some of the explanatory variables into a relatively limited number of categories, with the understanding that the quantities in the same category are strongly dependent on each other. A typical scenario is the existence of deterministic relationships governing the values of the variables in the same group.  Further, it is often the case that only a few of these groups  contribute meaningfully to the experimental observations, with the remaining variables being redundant or uninformative for predictive purposes. An important use-case of such structure is that of high dimensional supervised learning, where the explanatory variables are often clustered via natural constraints. For instance, meteorological measurements in spatially adjacent locations are likely to be highly correlated. Similarly, frequency of occurrence of certain words or phrases in spam emails are likely to be highly correlated.

A  different setting in which a group structure on variables plays a significant role is that of community detection problems, where groups or clusters pertain to connectivity patterns in an underlying network. A typical example is that of a social network, where connectivity corresponds to friendship or acquaintance; yet another is that of collaboration network among scientists \citep{fortunato2010community,fortunato2007resolution,reichardt2006statistical}.  Due to obvious practical ramifications,  this area has witnessed intense research activity in recent years,  a significant achievement of which is the extensive theory of \textit{Stochastic Block Models} (\textit{abbrv.} SBM) \citep{abbe2017community,abbe2015exact,bandeira2016low,goldenberg2010survey,holland1983stochastic,karrer2011stochastic}. In general, incorporating network geometry into standard statistical learning problems has been an area of recent interest \citep{hallac2015network, li2020graph, li2019prediction, li2020high}.

\subsection{Network geometry and its many avatars}
The network structure brings into  focus the geometric perspective on the clustering phenomenon,  that is underpinned by the metric  induced by the weighted graph distance in the network \citep{von2007tutorial}.  Intertwined with such geometry is the canonical dynamics  associated to it --  in a very general Riemannian geometric setting,  the metric structure gives rise to a Laplacian operator, which in turn serves as the generator of the so-called \textit{heat flow} dynamics on the underlying space the most canonical instantiation of a diffusion in this context.  These correspondences are classical in metric geometry and harmonic analysis
\citep{rosenberg1997laplacian,jost2008riemannian}.
 
Associated with the Laplacian geometry and diffusion is the canonical model in statistical physics referred to as the \textit{Gaussian Free Field} (\textit{abbrv.} GFF) (c.f.  \cite{sheffield2007gaussian,berestycki2015introduction,friedli2017statistical}), which has also emerged to be of independent interest as an important instance of \textit{Gaussian graphical models}  \citep{zhu2003semi,zhu2003combining,ma2013sigma,kelner2019learning,rasmussen2003gaussian}.   GFF-s complete the above picture from a statistical point of view,  by embedding the stochastic dependency structure of Gaussian random variables in the setup of the geometric structure and dynamical properties of a weighted networks.

\subsection{Our contributions} 
In the present work, we bring together these disparate strands  into synergy -- clustering phenomena on variables or predictors in one direction, those in SBM-type network models in another, and the  geometry of and dynamics on weighted networks in yet a third direction, along with their statistical physical implications. Leveraging their interplay,  we obtain an algorithm to perform effective regression analysis in both high and low dimensional setup for variables with a \textit{latent group structure},  using  \textit{limited and local access} to  underlying network dependencies.  In a more general setup, when an underlying graph may not be explicit in the problem description, we demonstrate a procedure to construct such a network based on the available data.  
From an algorithmic perspective,  our methodology alludes to interesting connections the with so-called \textit{diffusion mapping} techniques, which have been effective as dimension reduction tools \citep{coifman2005geometric, coifman2006diffusion}.    

More generally,  our approach opens the avenue to applications of similar diffusion-based techniques to classical statistical and data analytical problems, that are generally static in nature. The inherently local nature of the heat flow and related diffusion algorithms enables us to solve the relevant constrained optimization problems while being oblivious to the global geometry of the graph.  In addition to economies of computational resources,  such locality is of significance with regard to questions of privacy in data analysis, a problem that is gaining increasing salience in today's hyper-networked world. Finally, in a significant departure from standard network-based techniques (such as spectral clustering), our approach does not require prior knowledge of the number of groups or clusters. This makes it much more relevant in practical scenarios, where such information may not be readily available to the user.

In another direction, we note that our approach can be generalised to incorporate information on overlapping groups of variables (e.g., in a financial context, a stock can represent a company that belongs to or has significant exposure to several different market segments). In network data analysis, a prominent approach for modeling overlapping communities within the SBM framework is the Mixed-Membership Stochastic Block Model (MMSBM) \citep{airoldi2008mixed}. Laplacians of networks with such mixed-memberships can be readily integrated into our algorithm and  exploring the theoretical aspects of this combination is a compelling direction for follow-up study, which we hope to return to in the future.
 
\section{Lasso and its derivatives}

Variable selection is a classic problem in statistics, which has become all the more important in the present age with the routine availability of large scientific datasets with measurements on tens of thousands of variables. Sparsity has become a key methodological instrument for meaningful inference from such ``high-dimensional'' datasets. Parsimonious models are easier to interpret and more resistant to overfitting. The Least Absolute Shrinkage and Selection Operator (abbrv. \emph{lasso}) \citep{tibshirani1996regression}, which employs a sparsity inducing $\ell_1$-penalty, is perhaps the most prominent method of variable selection.

A major problem with vanilla lasso is that it treats all variables equally. Thus, when there are natural groups in the variables, some variables in a group can get kicked off the model with other members still included. The \emph{group lasso} penalty \citep{yuan2006model} aims to solve this issue.  Consider a supervised learning problem with $p$ predictors and corresponding parameter $\beta \in \R^p$. Denoting the groups by $\cC_1$, \ldots, $\cC_k$, the group lasso penalty uses a weighted $\ell_1$ norm of the groupwise $\ell_2$ norms:
\[
    \mathrm{GL}(\beta) = \sum_{\ell = 1}^k \sqrt{|\cC_\ell|} \,\, \|\beta_{\cC_\ell}\|_2.
\]
As a consequence, variables in a group exit the model together.

The group lasso penalty requires the groups to be known in advance.  However, in many practical scenarios, the group information is a priori unknown to the statistician. There has been some work to address this issue. The \emph{cluster representative lasso} (CRL) and \emph{cluster group lasso} (CGL) algorithms of \cite{buhlmann2013correlated}  perform an initial clustering of variables into groups and then use the estimated clusters for grouped variable selection.

However, a notable drawback of this approach is that when the group information is very noisy or the group structure is weak, then an initial clustering can produce completely erroneous group estimates, rendering the subsequent group lasso ineffective (e.g., in terms of prediction accuracy). 

In this paper, we construct a penalty that incorporates possible network information in the form of a Laplacian matrix and automatically selects variables in groups when there is an underlying strong group structure. Notably, we dispense with the elaborate pre-processing step involving clustering of the variables, spectral or otherwise, which can be computationally intensive. This also has the advantage that we do not need any knowledge of the number of groups in advance, and when there is no strong cluster structure present in the network, the proposed method does not perform any implicit clustering, thereby remedying the problem alluded to in the previous paragraph.

\section{Laplacian geometry of graphs and the heat flow penalty}\label{sec:setup}
Suppose that we have explanatory variables $X_1, \ldots, X_p$,  whose group structure is captured by a graph $G=(V,E)$ on $|V|=p$ vertices,  each vertex corresponding to an $X_i$.  In the simplest scenario,  the groups of variables would correspond to the connected components $\{\calC_i\}_{i=1}^k$ of $G$.  More generally,  we consider the nodes of $G$ to be a union of subsets  $\{\calC_i\}_{i=1}^k$ (corresponding to the variable groups),  such that the $\calC_i$-s are relatively densely intra-connected,  but the inter-connections across the $\calC_i$-s are relatively sparse. This is similar in flavour to the problem of multi-way partitioning of graphs,  which has attracted considerable interest over the years \citep{lee2014multiway}.

Let $L = D - A$ denote the Laplacian of $G$, where $A$ is the adjacency matrix of $G$ and $D$ is the diagonal matrix of vertex degrees, i.e. $D = \mathrm{diag}(A\bone)$, where $\mathbf{1}$ is the all ones vector in $\R^p$. We introduce the penalty
\begin{equation}\label{eq:pen-hflasoo}
    \Lambda_t(\beta) = \langle \Phi(e^{-tL}(\beta \odot \beta)), \mathbf{1} \rangle,
\end{equation}
where $\odot$ denotes Hadamard/elementwise product of vectors, $\Phi(\beta) = (\sqrt{|\beta_1|}, \ldots, \sqrt{|\beta_p|})^\top$.  
As we show in Lemma 1 and Corollary 2 in the appendix, in the limit of $t \to \infty$, the quantity $\Lambda_t(\beta)$ approaches the classical group lasso penalty $\mathrm{GL}(\beta)$ in the setting where the components $\{\calC_i\}_{i=1}^k$ are fully disconnected from each other in the graph $G$. A non-rigorous but intuitive argument is given in Section~\ref{sec:intuition} below.  See Figure~\ref{fig:pen-balls} for a demonstration of this convergence. 

In the other extreme, notice that for $t = 0$, we have
\[
    \Lambda_t(\beta) = \|\beta\|_1.
\]
Thus for finite $t$, $\Lambda_t$ interpolates between the lasso ($t = 0$) and group lasso ($t = \infty$). This suggests that one may treat $t$ as a tuning parameter to be selected via cross-validation. This is particularly useful if it is suspected that the group information in the graph $G$ is weak and alleviates the problem mentioned in the penultimate paragraph of the previous section: via cross-validation it is possible to essentially default to lasso when the group structure is weak.

Note that $\Lambda_t$ directly incorporates the Laplacian in terms of the heat flow operator $e^{-tL}$ on the underlying graph $G$. Unlike the group lasso penalty, $\Lambda_t$ is non-convex which is a potential problem vis-a-vis  optimization. However, using the connection with heat flow, we can calculate $\Lambda_t$ in an efficient manner via random walks on the graph $G$. 

We consider the penalised supervised learning problem 
\begin{equation} \label{eq:pen_loss}
    F_{t, \lambda}(\beta;X,y) = \calL(\beta;X,y)+ \lambda \Lambda_t(\beta),
\end{equation}
where $(X,y)$ denotes the training data,  and $\calL(\beta;X,y)$ is a suitable loss function defined with respect to the problem and computed using the training data at the parameter $\beta$.  We expect that for suitably large $t$, the solution of this optimization problem will be close to that of the classical group lasso problem, provided that the group structure in $G$ is strong. 

A fundamental example of this set-up is the penalised 
regression problem $\arg \min_{\beta} F_{t, \lambda}(\beta)$,  where
\[
    F_{t, \lambda}^{\text{reg}}(\beta) = \frac{1}{2n} \|y - X\beta\|_2^2 + \lambda \Lambda_t(\beta). 
\]
More generally,  an important family of examples  is accorded by the problem of penalised likelihood maximization,  where $\calL(\beta;X,y)$ is the negative log-likelihood of the observed data $(X,y)$ at the parameter value $\beta$.  Yet another instance is that of the so-called Huber's loss,  where the quadratic function of the $L_2$ norm in the regression set-up is replaced by a different convex function \citep{huber1992robust}.


\begin{figure}[!t]
    \centering
    \begin{tabular}{cc}
        \includegraphics[scale = 0.25]{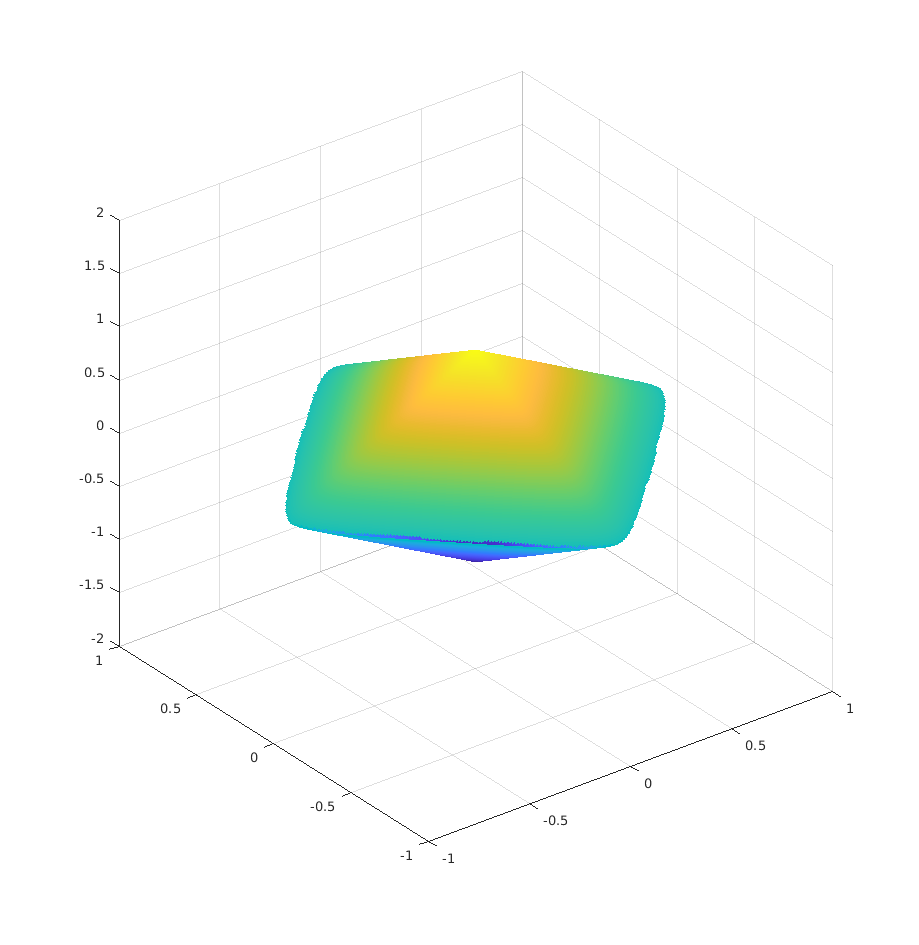} & \includegraphics[scale = 0.25]{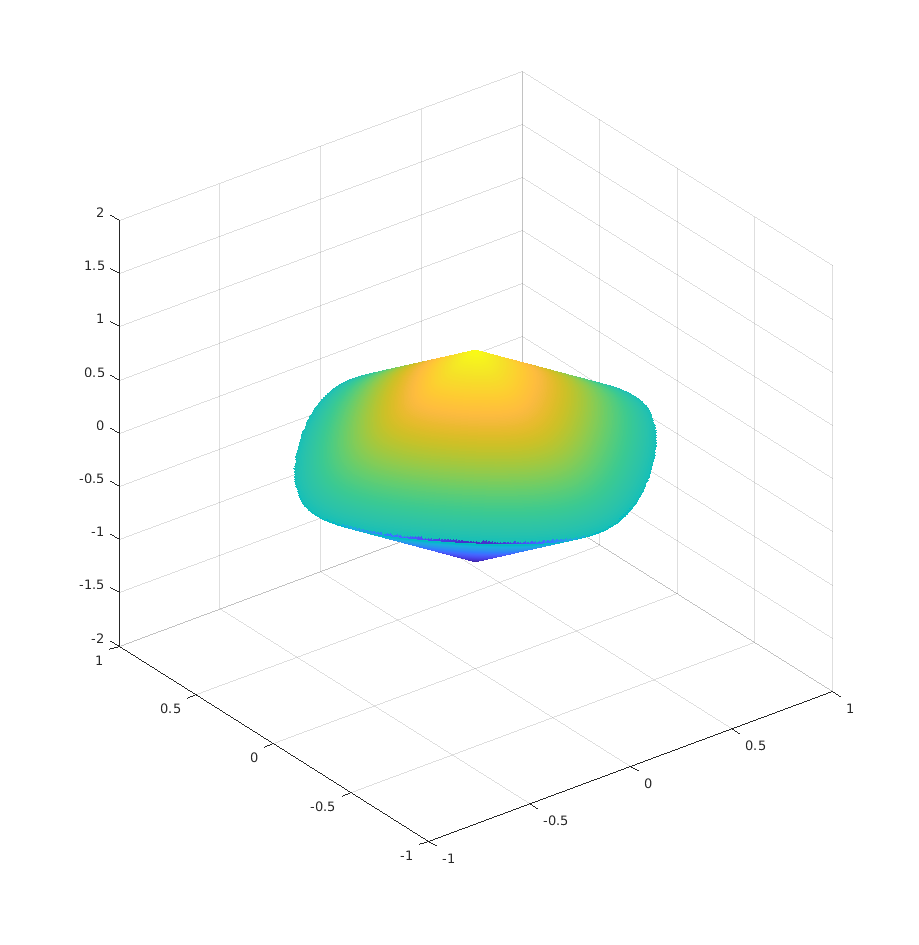} \\
        $t = 0.01$                         & $t = 0.1$ \\
         \includegraphics[scale = 0.25]{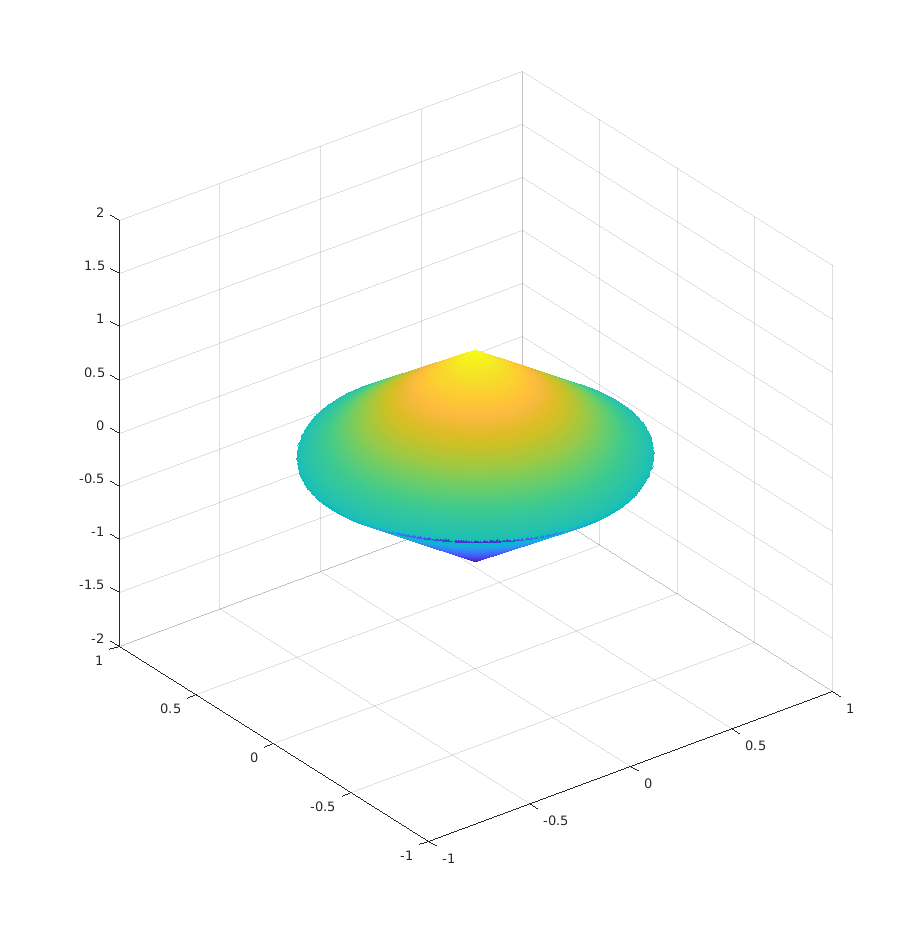} & \includegraphics[scale = 0.25]{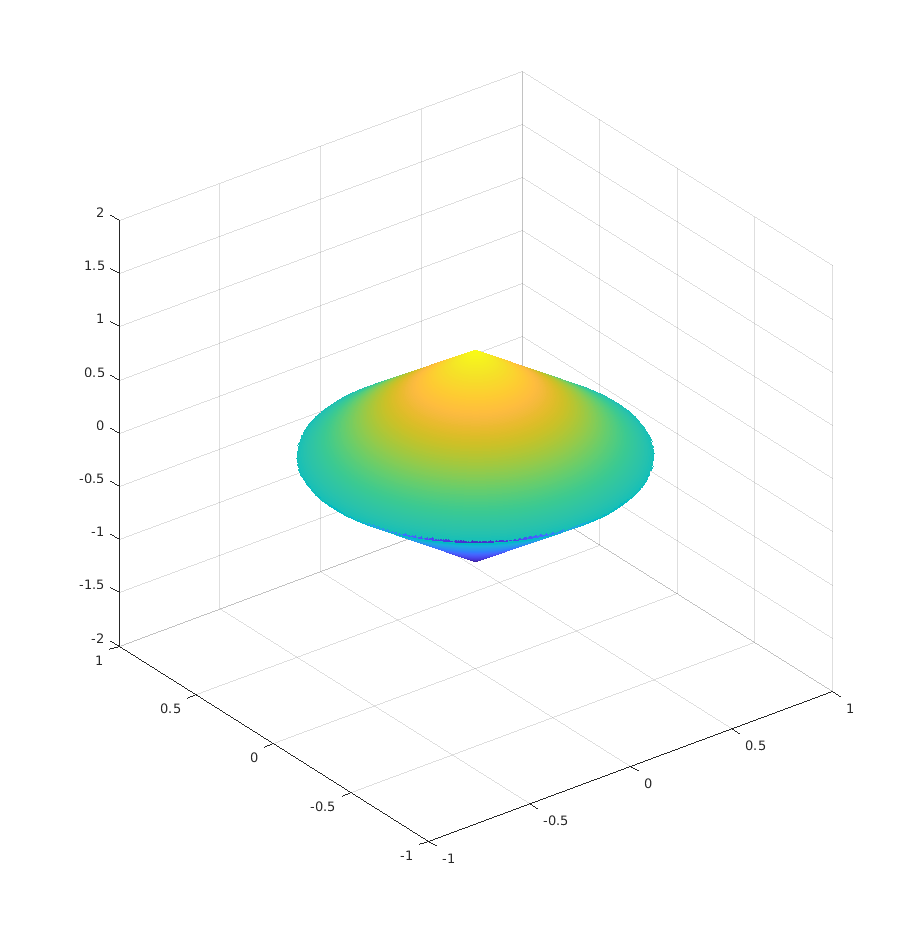} \\
         $t = 0.5$ & $t = \infty$
    \end{tabular}
    \caption{Consider three variables with two groups $\{1, 2\}$ and $\{3\}$. In this display, we plot of the sublevel set $\{\beta \mid \Lambda_t(\beta) \le 1\}$ for different values of $t$. The graph $G$ here is the union of an isolated vertex and an edge. The eigengap $\lambda_g = 2$. The bottom right plot corresponds to the Group Lasso penalty.}
    \label{fig:pen-balls}
\end{figure}

\subsection{Intuition behind the heat flow penalty}\label{sec:intuition}
We now present a non-rigorous argument elucidating the main intuition behind the proposal \eqref{eq:pen-hflasoo}. Consider a graph with exactly $k$ connected components $\cC_1, \ldots \cC_k$. The equilibrium measure of a continuous-time random walk $(Z_t)_{t \ge 0}$ started at a vertex $i \in \cC_j$ is the uniform distribution on $\cC_j$. Therefore, for sufficiently large $t$, one has that for any $i \in \cC_j$,
\[
    (e^{-tL}(\beta \odot \beta))_i = \E[(\beta \odot \beta)_{Z_t} \mid Z_0 = i]\approx \frac{1}{|\cC_j|}\sum_{l \in \cC_j} \beta_l^2 = \frac{\|\beta_{\cC_j}\|_2^2}{|\cC_j|}.
\]
It follows that
\begin{align*}
    \Lambda_t(\beta) &= \langle \Phi(e^{-tL}(\beta \odot \beta)), \mathbf{1} \rangle \\
                     &= \sum_{j = 1}^k \sum_{i \in \cC_j} \sqrt{(e^{-tL}(\beta \odot \beta))_i} \\
                     &\approx \sum_{j = 1}^k \sum_{i \in \cC_j} \frac{\|\beta_{\cC_j}\|_2}{\sqrt{|\cC_j|}} \\
                     &= \sum_{j = 1}^k  \sqrt{|\cC_j|} \,\, \|\beta_{\cC_j}\|_2 \\
                     &= \mathrm{GL}(\beta).
\end{align*}
We will later give a result making the above approximation precise.

Although, we have used the continuous-time random walks for formulating the penalty \eqref{eq:pen-hflasoo}, it is possible to use discrete-time random walks to devise a similar but slightly less elegant penalty. See Section~\ref{sec:cont-vs-discrete} in the Appendix for details.

\section{Local heat flow dynamics and an algorithm for implicit group sparsity} \label{sec:algo}

In this section, we describe our algorithms for minimising the objective function in \eqref{eq:pen_loss}. See Figure~\ref{fig:alg-flowchart} for a dependency graph of the various algorithms.

\algrenewcommand\algorithmicrequire{\textbf{Input:}}
\algrenewcommand\algorithmicensure{\textbf{Output:}}
\newcommand{\old}{\mathrm{old}}
\newcommand{\new}{\mathrm{new}}
\newcommand{\hflow}{\textsc{heatflow}}
\newcommand{\subgrad}{\mathrm{subgrad}}
\newcommand{\iter}{i}
\newcommand{\maxiter}{N}
\newcommand{\reldiff}{\mathrm{reldiff}}
\newcommand{\mean}{\mathrm{mean}}

\subsection{A heat-flow-based algorithm}
Our general algorithmic framework is underpinned by a heat-flow-based computation of vectors of the form $e^{-tL} f$ (or a subset of its co-ordinates) for a given $f \in \R^p$, where we view the latter as a function on the nodes of the graph $G$. In this setting, we observe that 
\[
    (e^{-tL} f)_i = \E[f_{Z_t} \mid Z_0 = i], 
\]
where $(Z_t)_{t \ge 0}$ is a continuous-time simple random walk on $G$, which is the canonical analogue of heat flow dynamics in this setup. This indicates that if we start $B$ random walks $Z^{(1)}, \ldots, Z^{(B)}$ from the vertex $i$, then an estimate of $(e^{-tL} f)_i$ would be
\[
    \widehat{(e^{-tL} f)}_i = \frac{1}{B}\sum_{j = 1}^B f_{Z^{(j)}_t}.
\]

In Algorithm~\ref{alg:heatflow}, we describe the pseudo-code for simulating $B$ heat flows from each vertex, run till time $t$. Note that the for-loops can be easily paralellised, rendering this algorithm highly efficient. The running time (i.e. computational complexity) for Algorithm~\ref{alg:heatflow}, even without factoring in parallelisation, is only $O(pB\Ns)$, where $\Ns$ is the step count required for a single run of the heat flow. We demonstrate in \eqref{eq:eq:step_well_clustered} that $\Ns$ is typically $O(\max\{\log p, \log n\})$; as such the computational complexity for Algorithm~\ref{alg:heatflow} is only $O(B \cdot p \cdot \max\{\log p, \log n\})$. Algorithm~\ref{alg:heatflow} outputs a $p \times B$ matrix $H$ whose $(i, j)$-th entry stores the state of the $j$-th heat flow (started at vertex $i$) at time $t$. We emphasize that Algorithm~\ref{alg:heatflow} is run just once and as such it may be thought of as a pre-computation step.

The pseudo-code for approximating $e^{-tL} f$ (or a subset of its co-ordinates) using heat flow is presented in Algorithm~\ref{alg:heatflow_on_vector}. This uses the heat flow matrix $H$ generated in Algorithm~\ref{alg:heatflow}. Finally, the complete subgradient-based optimisation procedures are presented in Algorithms~\ref{alg:sd} and \ref{alg:block_cd}.


\begin{figure}
    \centering
    \includegraphics[width = 1.1\textwidth]{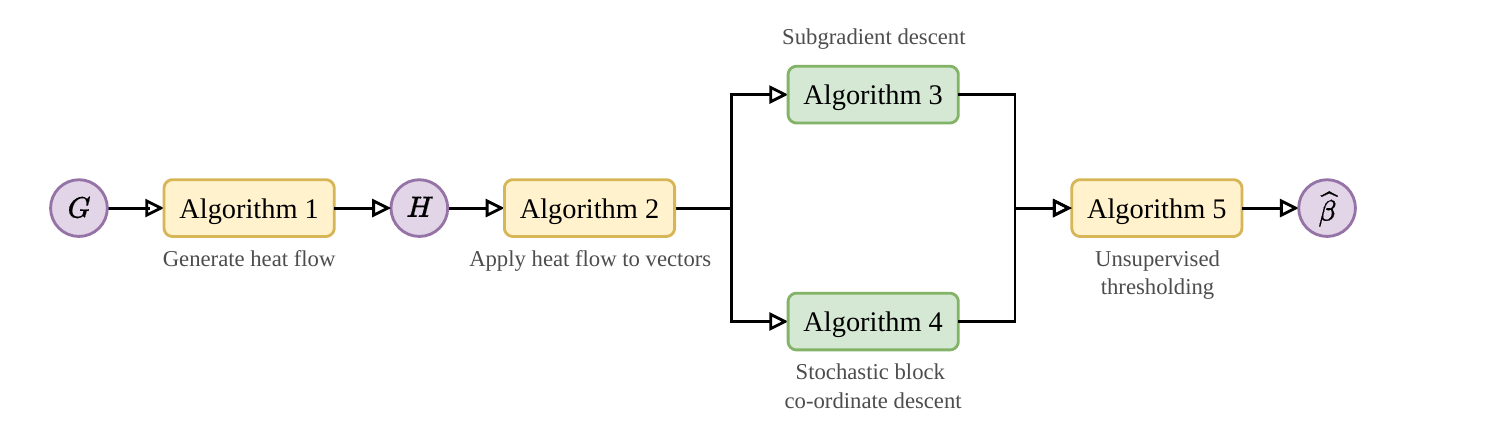}
    \caption{The dependency graph of Algorithms~\ref{alg:heatflow}-\ref{alg:threshold}. Algorithms~\ref{alg:sd} and \ref{alg:block_cd} are two alternative gradient-descent-type optimisation algorithms. We can either have blackbox access to $G$ for simulating local random walks, or can generate $G$ in a data-driven approach via Algorithm \ref{alg:graph_estimation}.}
    \label{fig:alg-flowchart}
\end{figure}


\begin{algorithm}[!ht]
    \caption{Simulate heat flow}\label{alg:heatflow}
    \begin{algorithmic}[1]
        \Require Graph $G$; time $t$ till which to run the heat flow dynamics; $B$, the number of heat flows generated per vertex
        \Ensure A $p \times B$ matrix $H$, with $H_{ij}$ storing the state of the $j$-th heat flow started at vertex $i$ at time $t$.
        \For{$i = 1, \ldots, p$}
        \For{$j = 1, \ldots, B$}
        \State $s \gets 0$
        \State $H_{ij} \gets i$
        \If{$\mathrm{degree}(H_{ij}) > 0$}
            \While{$s < t$}
            \State Generate $\mathcal{E} \sim \mathrm{Exponential}(\mathrm{degree}(H_{ij}))$
            \State $s \gets s + \mathcal{E}$
            \State $H_{ij} \gets $ a neighbour of $H_{ij}$, chosen at uniformly at random 
            \EndWhile
        \EndIf
        \EndFor
        \EndFor
    \end{algorithmic}
\end{algorithm}

\begin{algorithm}[!ht]
    \caption{Heat flow on a vector $f$}\label{alg:heatflow_on_vector}
    \begin{algorithmic}[1]
        \Require Heat flow matrix $H$; vector $f$, indices $S$ at which value required
        \Ensure $g = \textsc{heatflow}(f, S)$, an estimate of $(e^{-tL}f)_S$
        \For{$i \in S$}
        \State $g_i \gets \frac{1}{B}\sum_{j = 1}^B f_{H_{ij}}$     
        \EndFor
    \end{algorithmic}
\end{algorithm}

\subsection{Subgradient and stochastic block co-ordinate descent}
We use subgradient descent methods to minimise the penalised loss $F_{t,\la}(\beta;X, y)$ in \eqref{eq:pen_loss}, which we denote henceforth as $F_{t,\la}(\beta)$ for reasons of brevity. The intermediary computation of the subgradients and losses is performed via a heat flow based local dynamics on the network. The precise algorithmic implementation of the latter is given in Algorithms \ref{alg:heatflow} and \ref{alg:heatflow_on_vector}, whereas the complete optimisation procedures (using these as subroutines) are encapsulated in Algorithms \ref{alg:sd} and \ref{alg:block_cd}.

\subsubsection{Subgradient descent}\label{sec:subgradient-descent}
We now describe how the heat flow operator appears in the subgradient computations. Let us compute the subgradient of $F_{t, \lambda}(\beta)$. Set $h = e^{-tL} (\beta \odot \beta)$. Then $\Lambda_t(\beta) = \sum_{j = 1}^p \sqrt{|h_j|}$. Write $s(x) := \sqrt{|x|}$ so that $\partial s(x) = \frac{\sgn(x)}{2 s(x)}$. Then
\[
    \frac{\partial \Lambda_t(\beta)}{\partial \beta_{\ell}} = \sum_{j = 1}^p \partial s(h_j) \frac{\partial h_j}{\partial \beta_{\ell}} = \sum_{j = 1}^p \partial s(h_j) 2 (e^{-tL})_{j \ell} \beta_{\ell} = \sum_{j = 1}^p \ell(h_j) (e^{-tL})_{j \ell} \beta_{\ell},
\]
where $\ell(x):= 2 \partial s(x) = \sgn(x)/\sqrt{|x|}$. Define the vector
\[
    \zeta \equiv \zeta(\beta) := (\ell(h_j))_{1 \le j \le p}.
\]
Then, using the fact that $e^{-tL}$ is symmetric, we may write
\begin{equation}\label{eq:Lambda_gradient}
    \nabla_{\beta} \Lambda_t(\beta) = (e^{-tL} \zeta) \odot \beta.
\end{equation}
Therefore
\[
    \nabla_{\beta} F_{t, \lambda}(\beta) = \frac{1}{n}X^\top X \beta - \frac{1}{n} X^\top y + \lambda (e^{-tL} \zeta) \odot \beta.
\]
Thus the general subgradient step would be
\[
    \beta^{(m + 1)} = \beta^{(m)} - \alpha^{(m)} \partial F_{t, \lambda}(\beta^{(m)}),
\]
for some learning rare $\alpha^{(m)}$.

\begin{algorithm}[!ht]
    \caption{Subgradient descent}\label{alg:sd}
    \begin{algorithmic}[1]
        \Require Heat flow matrix $H$; initial estimate $\beta^{(0)}$; $\nabla \calL$, gradient of loss function; $\lambda$, penalty parameter; $\epsilon$, error tolerance; $\maxiter$, maximum number of subgradient steps; $\alpha: \N \rightarrow (0, \infty)$, learning rate protocol
        \Ensure $\hat{\beta}$, an approximate local minimum of heat flow penalised objective
        \State $\beta^{(\old)} = \beta^{(0)}$
        \State $\reldiff \gets 2\epsilon$
        \State $\iter \gets 0$
        \While{$\reldiff > \epsilon$ \textbf{and} $\iter \le \maxiter$}
        \State $\iter \gets \iter + 1$
        \State $h \gets \hflow(\beta_{\old} \odot \beta_{\old}, [p])$
        \State $\zeta \gets (\ell(h_j))_{1 \le j \le p}$ \Comment{$\ell(x) = \sgn(x)/\sqrt{|x|}$} 
        \State $\subgrad \gets \nabla \calL(\beta^{(\old)}) + \lambda \,\, \hflow(\zeta, [p]) \odot \beta^{(\old)}$
        \State $\beta^{(\new)} \gets \beta^{(\old)} - \alpha(\iter) \,\, \subgrad$
        \State $\reldiff \gets \frac{\|\beta^{(\new)} - \beta^{(\old)}\|}{\|\beta^{(\old)}\|}$
        \EndWhile
    \end{algorithmic}
\end{algorithm}


\subsubsection{Stochastic block co-ordinate descent}
We may also use stochastic block co-ordinate descent updates. For this we need to compute $(e^{-tL} \zeta)_S$ for some $S \subset [p]$.

In our particular implementation, at each run we choose $q$ random co-ordinates to be updated (i.e. $|S| = q$). This update may be efficiently accomplished via an on-demand computation of only the required co-ordinates of $h$, as described in Algorihm~\ref{alg:block_cd}. We now describe this on-demand computation. We need to compute $((e^{-tL}\zeta) \odot \beta)_S$. We do this in the following way. For any $j \in S$, we let $\mathcal{E}_j = H_{j, \star}$, i.e. $\mathcal{E}_j$ is a $B \times 1$ vector storing the end-points of the $B$ random walks started at $j$. We compute $h = e^{-tL} (\beta \odot \beta)$ at these co-ordinates only. For this we re-use the heat flow matrix $H$. We then set $\zeta_{\mathcal{E}_j} = \ell(h_{\mathcal{E}_j})$. Since $\mathcal{E}_j$ consists of the end-states of $B$ random walks started at $j$, therefore $\mean(\zeta_{\mathcal{E}_j})$ can be taken as a good estimate of $(e^{-tL} \zeta)_j$ via a Monte Carlo approximation. (Here, for a vector $x = (x_1, \ldots, x_l)^\top$, $\mean(x) := \frac{1}{l}\sum_{j = 1}^l x_j$). 


\bk

\begin{algorithm}[!ht]
    \caption{Stochastic block co-ordinate descent}\label{alg:block_cd}
    \begin{algorithmic}[1]
        \Require Heat flow matrix $H$; initial estimate $\beta^{(0)}$; $\nabla \calL$, gradient of loss function; $\lambda$, penalty parameter; $\epsilon$, error tolerance; $\maxiter$, maximum number of subgradient steps; $\alpha: \N \rightarrow (0, \infty)$, learning rate protocol; $q$, block size for co-ordinate updates
        \Ensure $\hat{\beta}$, an approximate local minimum of heat flow penalised objective
        \State $\beta^{(\old)} = \beta^{(0)}$
        \State $\reldiff \gets 2\epsilon$
        \State $\iter \gets 0$
        \While{$\reldiff > \epsilon$ \textbf{and} $\iter \le \maxiter$}
        \State $\iter \gets \iter + 1$
        \State $S \gets \mathrm{sample}([p], q)$ \Comment{Only these $q$ random coordinates will be updated in this run}
        \For{$j \in S$}
        \State $\mathcal{E}_j \gets H[j, ]$ 
        \State $h_{\mathcal{E}_j} \gets \hflow(\beta^{(\old)} \odot \beta^{(\old)}, \mathcal{E}_j)$
        \State $\zeta_{\mathcal{E}_j} \gets \ell(h_{\mathcal{E}_j})$ \Comment{$\ell$ is applied co-ordinatewise}
        \State $g_j \gets \mean(\zeta_{\mathcal{E}_j}) \beta^{(\old)}_j$ 
        \EndFor
        \State $\subgrad \gets \nabla_S\calL(\beta^{(\old)}) + \lambda \,\, g_S$
        \State $\beta^{(\new)}_S \gets \beta^{(\old)}_S - \alpha(\iter) \,\, \subgrad$
        \State $\reldiff \gets \frac{\|\beta^{(\new)} - \beta^{(\old)}\|}{\|\beta^{(\old)}_S\|}$
        \EndWhile
    \end{algorithmic}
\end{algorithm}

\subsection{A final unsupervised thresholding step}
We apply a unsupervised thresholding step to the output $\hat{\beta}$ of the optimisation to further enhance its support recovery properties. Since the estimate of $\beta_j$ is expected to be small for $j$ outside the support of $\beta$, we may envisage the entries $|\hat{\beta}|$ forming two clusters, one of them being close to $0$. We therefore apply $K$-means clustering with $K = 2$ on $|\hat{\beta}|$. Let $C_*$ be the cluster whose mean is the farthest from $0$. We only keep the values of $\hat{\beta}_{C_*}$ and zero out the rest of the entries. The pseudo-code of this procedure appears in Algorithm~\ref{alg:threshold}. 

We note that since the above procedure uses a unsupervised clustering step instead of hard-thresholding, we do not need to select any threshold parameter.


\begin{algorithm}[!ht]
    \caption{Unsupervised thresholding}\label{alg:threshold}
    \begin{algorithmic}[1]
        \Require $\hat{\beta}$, output of subgradient descent and/or stochastic block co-ordinate descent
        \Ensure $\hat{\beta}_{\mathrm{thres}}$, a thresholded version of $\beta$
        \State $(C_1, m_1), (C_2, m_2) \gets \textsc{kmeans}(|\beta|, 2)$ \Comment{$C_i$ is the $i$-th cluster and $m_i$ is its mean}
        \State $j_* \gets \arg \max_{j \in \{1, 2\}} m_j $ \Comment{$m_{j_*}$ is the cluster mean farthest from $0$}
        \State Set $C_* := C_{j_*}$.
        \For{$i = 1, \ldots, p$}
            \If{$i \in C_{*}$}
                \State $\hat{\beta}_{\mathrm{thres}, i} = \hat{\beta}$
            \Else
                \State $\hat{\beta}_{\mathrm{thres}, i} = 0$.
            \EndIf
        \EndFor
    \end{algorithmic}
\end{algorithm}

\subsection{Learning the graph Laplacian from data}
If the Laplacian $L$ (equivalently, the graph $G$) is not a priori available, we may estimate it using the covariance structure of the explanatory variables as follows (the pseudo-code appears in Algorithm~\ref{alg:graph_estimation}). Let $\hat{\Sigma}$ be some estimate of the true covariance matrix $\Sigma$. Let $\hat{R}$ denote the estimate of the true correlation matrix, obtained by rescaling $\hat{\Sigma}$. We construct an adjacency matrix by thresholding $\hat{R}$:
\[
    A_{ij} = \bone_{\{|\hat{R}_{ij}| \ge \tau(\hat{R})\}},
\]
where $\tau(\hat{R})$ is an appropriate quantile of the $|\hat{R}_{ij}|$'s. (E.g., in our simulation experiments, we use the $0.75$-th quantile.) We then take the unnormalised Laplacian corresponding to $A$ as an estimate of $L$.

As for estimating $\Sigma$, we may use the sample covariance matrix in low to moderate dimensions, and appropriately thresholded versions of it (e.g., \cite{bickel2008covariance, bickel2008regularized}) or other shrinkage estimators in high dimensions (e.g., \cite{friedman2008sparse, chen2010shrinkage}).

\begin{algorithm}[!ht]
    \caption{Graph estimation}\label{alg:graph_estimation}
    \begin{algorithmic}[1]
        \Require $\hat{R}$, some estimate of the population correlation matrix of the covariates; $\alpha$, quantile at which to threshold
        \Ensure $A$, adjacency matrix of a graph constructed from $\hat{R}$
        \State $\theta \gets \mathrm{quantile}(|\hat{R}|, \alpha)$ \Comment{$|\hat{R}| = ((|\hat{R}_{ij}|))$}
        \State $A \gets \bzero_{p \times p}$
        \For{$i = 1, \ldots, (p - 1)$}
        \For{$j = (i + 1), \ldots, p$}
        \State $A_{ij} \gets \ind_{\{|\hat{R}_{ij}| > \theta\}}$
        \State $A_{ji} \gets A_{ij}$ 
        \EndFor
        \EndFor
    \end{algorithmic}
\end{algorithm}

\subsection{Optimisation via local network dynamics}

A cornerstone of our algorithms is that we do not require any direct knowledge of the group structure. Clustering algorithms for estimating the group structure typically require the knowledge of the number of groups. Thus in order to use group Lasso with an estimated group structure, one \emph{needs to know the number of groups}. Our algorithms, on the other hand, do not have such a requirement.

Further, we do not even require access to the full network. Rather, at each step of the optimisation algorithm, we find an approximation to the subgradient of the penalty (at the current value of the parameter $\b$) by using an Monte Carlo based approach. Additional economy in computational resources is accorded by the fact that we are able to use the same heat flow throughout the optimization procedure, the initial and terminal nodes of which generated and stored for successive calls. 

In particular, running a continuous time random walk only entails exploring a local neighbourhood of the current state, which relieves us of the necessity to work with the entire graph at one go. The latter would be necessary, e.g., in an approach where we wanted to do an initial spectral clustering  in order to arrive at a detailed understanding of the group structure  \citep{buhlmann2013correlated}. This may be prohibitively expensive in the setting of real-world massive networks, such as the large-scale social networks or the world-wide web.

In fact, we only  require \textit{oracle access} to a black box that returns the terminal state of the heat flow starting from a prescribed initial node. Such limited and local  access to the network data can have significant implications with regard to considerations of privacy and security, which play an increasingly important role in modern statistical research. 

\subsection{The duration and  step count for the heat flow dynamics}
The duration of the heat flow is determined by the goal to make the difference between our penalty and the classical group lasso penalty small,  a bound on which is accorded by Theorem \ref{thm:prediction-consistency}. It follows therefrom  that,  under reasonable conditions,  it suffices to have $\tflow \gg \frac{1}{\la_g}\max \{ \log n, \log p\}$.  In the most important setting of the clusters being fully disconnected from each other but densely connected within each other,  the ground state $\la_g$ is the minimum of the ground states for the individual clusters. For the components being generic densely connected graphs of size $\Theta(p)$, we typically have $\la_g=\Theta(p)$,  whence $\tflow \gg \frac{1}{p}\max \{ \log n, \log p\}$; we refer to Section 3 in the appendix for details. We note that although the optimal choice of $t_{\mathrm{flow}}$ depends on $\lambda_g$, in practice we can perform \emph{cross-validation} to choose the optimal value of $t_{\mathrm{flow}}$. It is precisely for this reason that our algorithms do not need the knowledge of the number of groups. 

The number of steps of the heat flow dynamics is,  roughly speaking,  the heat flow time $\tflow$ times the average number of steps per unit time.  If the current state of the dynamics is a node $v \in G$ with degree $\mathrm{deg}(v)$,  then the time until the next step in the random walk is an $\mathrm{Exponential}(\mathrm{deg}(v))$) random variable.  Thus,  step time distribution is stochastically dominated by an $\mathrm{Exponential}(d_{\max})$ random variable,  where $d_{\max}$ is the maximum degree of $G$. Thus,  on average,  the total step count for the heat flow dynamics is given by $\Ns=O(d_{\max} \cdot \tflow)=O(d_{\max} \cdot \frac{1}{\la_g} \cdot \max \{ \log n, \log p\})$.  For typical strongly intra-connected components as above, we have $d_{\max}=\Theta(p)$ (for details we refer to Section 3 in the appendix), whence we have $\Ns = O(\max \{ \log n, \log p\})$.  

The logarithmic dependence on the problem dimensions as indicated above entails a light computational load for the heat flow based dynamic algorithm. 

\subsection{A spectral perspective and the role of non-convexity}

We note in passing that the spectral data of a heat flow operator (equivalently,  an appropriate random walk transition matrix) is known to be useful for learning tasks,  especially in the context of diffusion mapping for dimensionality reduction problems \citep{coifman2005geometric, coifman2006diffusion}.  However,  the random walk in diffusion mapping takes place in a different space -- namely,  on the space of the actual data points (therefore,  having $n$ nodes); whereas  in our setting,  the random walk takes place on the co-ordinate indices of the data points (thereby entailing $p$ nodes).  Nonetheless,   it would be of  interest to explore the possible interaction of these two diffusion-based approaches,  in particular regarding the possibility of incorporating a dimension reduction step in our paradigm to achieve further economy of computational resources.

It may be noted that the map $\Phi$ that is embedded in our penalty is mildly non-convex, with an algebraically explicit square-root structure. This may be compared to the setting of classical group lasso when the groups are a priori not known, where  non-convexity enters via its role in estimating the groups. This step involves solving a  clustering problem, which is generally non-convex.  In fact, the non-convex structure of the group recovery problem is known for its notoriety as a challenging optimization issue, and its rather intractable combinatorial nature makes it arguably a more complicated endeavour than the simple, algebraically explicit non-convexity posed by the map $\Phi$ in our approach. Exploiting this simple algebraic structure, perhaps by carrying out the  optimization in a different co-ordinate system, would be an interesting direction for future research.

\subsection{No prior knowledge of the number or structure of the groups}
A major advantage of the proposed approach is that we do not assume knowledge of the number or the structure of groups. If there is a sufficiently strong group structure, and a priori knowledge of $k$ is avaiable, off-the-shelf clustering techniques may be leveraged to produce good estimates of the group memberships and then one may directly employ Group Lasso based on this estimated group structure. This is precisely the approach taken in \cite{buhlmann2013correlated}. However, in absence of such knowledge, clustering methods may become unreliable, rendering the downstream group lasso pipeline ineffective. Since our approach does not rely on group structure estimation, we do not suffer from this issue. In fact, in such a scenario, one may tune parameter $t$ appropriately to achieve superior prediction performance. Since at $t = 0$, the proposed penalty reduces to the lasso penalty, by tuning $t$ we can essentially default to lasso (resp. group lasso) when then group structure present is weak (resp. strong), thereby combining the strengths of both lasso and group lasso.

\section{Theoretical guarantees}\label{sec:main}
We first set some notations in order to lay out our main theoretical results. For definiteness, we focus on the setting of regression
\[
    y = X\beta^* + \varepsilon,
\]
with a group sparsity structure on the parameters $\beta^*$.
However, we note in passing that similar analyses would apply to a wide range of applications of our method, including logistic regression, generalised linear models and other use cases.

Consider the situation where the graph $G$ has exactly $k$ connected components. Then it is a well-known fact that the spectrum $0 = \la_0 \le \la_1  \le \cdots  \le \la_p$ of the Laplacian matrix $L$ of $G$ then has exactly $k$ many zero eigenvalues. Let $\lambda_g := \min_{i > k} \lambda_i$ denote the \textit{spectral gap} of $L$. Let
\[
    A(\beta) = \{ i \mid 1 \le i \le k, \|\beta_{\cC_i}\|_2 \ne 0\},
\]
and
\[
    I(\beta) = \cup_{i \in A(\beta)} \cC_i.
\]

Note that $\Lambda_t(\beta)$ is a non-convex penalty. As such, the objective function $F_{t, \lambda}$ can have multiple local optima. However, we will show that in an appropriate neighbourhood of the true parameter value $\beta^*$, the minimum of $F_{t, \lambda}(\beta)$ approximately minimises the group Lasso penalty for sufficiently large $t$.

As we shall see below, the presence of an additional \textit{restricted eigenvalue property} leads to improved guarantees on the accuracy of our procedure, so we state this property below.

\medskip
\textit{\textbf{Property $\boldsymbol{\RE_{\gamma}(s, \kappa)}$}}: Let $\gamma \ge 0$. We say that the restricted eigenvalue property $\RE_{\gamma}(s, \kappa)$ holds for a matrix $M$ if
\begin{align*}
    \min\bigg\{\frac{\|M\Delta\|}{\sqrt{\|\Delta_A\|}} \, :\, |A| \le s, \quad \Delta \in \R^p \setminus \{0\}, \, \sum_{j \notin A} \sqrt{|\calC_j|} \|\Delta^j\| \le \gamma \sum_{j \in A} \sqrt{|\calC_j|} \|\Delta^j\|\bigg\} \ge \kappa,
\end{align*}
where $\Delta^j$ is a shorthand for $\Delta_{\cC_j}$ and $\Delta_A \equiv \Delta_{\cup_{j \in A} \cC_j}$. Here the parameter $\kappa$ might depend on $\gamma$ and $s$.
\medskip

Let
$B[\beta^*; \epsilon] := \{ \beta \mid n^{-1/2}\|X(\beta - \beta^*)\|_2 \le \epsilon\}$
and set
\[
    \Lambda(X; \eta) := \max_{1 \le j \le k} \sqrt{\bigg\|\frac{1}{n}X_j^\top X_j\bigg\|_{\op}}\bigg(1 +  \sqrt{\frac{4\log \eta^{-1}}{|\calC_j|}}\bigg).
\]
Also, let $\|\beta\|_{2, 1} := \sum_{j = 1}^k \|\beta^j\|_2$. Further, set $T_j := T_j$, $T_{\max} := \max_{1 \le j \le k} |\cC_j|$ and $T_{\min} := \min_{1 \le j \le k} |\cC_j|$.
We may then state:

\begin{theorem}[Estimation and prediction error]\label{thm:prediction-consistency}
    Let $\eta \in (0, 1)$. Suppose that $\lambda \ge \frac{2\sigma}{\sqrt{n}} \Lambda(X; \eta)$, and for $0 < \epsilon \le \epsilon_0$, let $\hat{\beta}_{t, \lambda} = \hat{\beta}_{t, \lambda}^{(\epsilon)}$ be the minimiser of $F_{t, \lambda}(\beta)$ in $B[\beta^*, \epsilon]$.
\begin{enumerate}
    \item [(a)] (Bounds on prediction error) Let $8 \sqrt{T_{\max}} \lambda \|\beta^*\|_{2, 1} \le \epsilon$.
    Then  with probability at least $1 - 2k\eta$, we have
    \begin{equation} \label{eq:PC_slow}
        \frac{1}{n}\|X(\hat{\beta}_{t, \lambda} - \beta^*)\|_2^2 = O(\|\beta^*\|_{2, 1} \lambda \sqrt{T_{\max}} + p e^{-t\lambda_g/2}).
    \end{equation}
    If we further assume that for some $\varpi > 0$, a restricted eigenvalue property $\RE_{3 + \varpi}(s, \kappa)$ holds for $X/\sqrt{n}$ for $s \ge |A(\beta^*)|$, then
    \begin{equation} \label{eq:PC_fast}
        \frac{1}{n}\|X(\hat{\beta}_{t, \lambda} - \beta^*)\|_2^2 = O\bigg(\frac{s\lambda^2 T_{\max}}{\kappa^2} + p e^{-t\lambda_g/2}\bigg).
    \end{equation}
\item [(b)] (Bound on estimation error) Assume that for some $\varpi > 0$, $X/\sqrt{n}$ satisfies $\RE_{3 + \varpi}(s, \kappa)$ for $s \ge |A(\beta^*)|$. Then with probability at least $1 - 2k\eta$, we have 
    \begin{equation}\label{eq:estimation-error}
        \|\hat{\beta}_{t, \lambda} - \beta^*\|_{2, 1} = O\bigg( \frac{\lambda s T_{\max}}{\kappa^2 \sqrt{T_{\min}}} + \frac{1}{\sqrt{T_{\min}}} \frac{p e^{-\lambda_g t / 2}}{\lambda}\bigg).
    \end{equation}
    \end{enumerate}
\end{theorem}

Note that the second term in \eqref{eq:estimation-error} may be made small by choosing $t$ sufficiently large.
(e.g., $t = \Omega(\mathrm{poly log}(n \vee p))$).

Now we will state our main result on support recovery. We will make the following assumptions (see Assumption~5.1 of \cite{lounici2011oracle}).
Let $\Psi[j, j'] = \frac{1}{n} X_j^\top X_{j'}$. We assume that there exists an integer $s \ge 1$ and some constant $\alpha > 1$ such that for any $1 \le j, j' \le k$,
\begin{align}\label{eq:def-phi}
    |(\Psi[j, j])_{t, t}| &=: \phi; \\
    \max_{t \ne t'} |(\Psi[j, j])_{t, t'}| &\le \frac{\sqrt{T_{\min}} \phi}{14 \alpha  \sqrt{T_{\max} s}} \frac{1}{\sqrt{T_j T_{j'}}}; \\
    \max_{1 \le t \le T_j, 1 \le t' \le T_{j'}, t \ne t'} |(\Psi[j, j'])_{t, t'}| &\le \frac{\sqrt{T_{\min}} \phi}{14 \alpha  \sqrt{T_{\max} s}}; \\
    \max_{t \ne t'} |(\Psi[j, j])_{t, t'}| &\le \frac{\sqrt{T_{\min}} \phi}{14 \alpha  \sqrt{T_{\max} s}} \frac{1}{\sqrt{T_j T_{j'}}}. 
\end{align}
Also, let
\begin{equation}\label{eq:def-c}
    c = \frac{3}{2} + \frac{16}{7(\alpha - 1)}.
\end{equation}

The following result shows that clipping the coordinates of $\hat{\beta}_{t, \lambda}$ at a suitable threshold exactly recovers the true support of $\beta^*$. 
\begin{theorem}[Support recovery]\label{thm:support-recovery}
    Suppose that the assumptions in the statement of Theorem~\ref{thm:prediction-consistency}-(b) hold.
    Assume further that for any $j \in A(\beta^*)$, we have
    \begin{equation}\label{eq:min_beta}
        \min_{i \in \cC_j} |\beta^*_i| > \frac{2 c}{\phi} \lambda \sqrt{T_{\max}}.
    \end{equation}
    Then  with probability at least $1 - 2k\eta$, we have
    \[
        \max_{j \notin A(\beta^*)} \max_{i \in \cC_j}|(\hat{\beta}_{t, \lambda})_i| \le \frac{c}{\phi} \lambda \sqrt{T_{\max}} < \min_{j \in A(\beta^*)} \min_{i \in \cC_j} |(\hat{\beta}_{\lambda, t})_i|,
    \]
    so that thresholding $\hat{\beta}_{t, \lambda}$ at level $\frac{c}{\phi} \lambda \sqrt{T_{\max}}$ recovers the support of $\beta^*$ perfectly.
\end{theorem}

We note here that we need the condition in \eqref{eq:min_beta} because we do not assume knowledge of the groups. If one knew the groups, a weaker groupwise lower bound assumption would have sufficed (see, e.g., Eq. (5.3) in \cite{lounici2011oracle}). 
We note that both $c, \phi$ are $\Theta(1)$ quantities. Also, under reasonable scenarios (e.g., random designs with $p/n$ bounded above), $\Lambda(X; \eta) = \Theta(1)$. Thus
\[
    \frac{c}{\phi} \lambda \sqrt{T_{\max}} = \Omega\bigg( \frac{\sigma\sqrt{T_{\max}}}{\sqrt{n}}\bigg).
\]
Thus, when the number of observations is large compared to the maximum group size, the above threshold is small.

We further demonstrate that the $\RE_{\gamma}(s, \kappa)$ property above for a random design matrix $X$ can be related, in the crucial setting of Gaussian covariates, to a similar property for the \textit{deterministic} covariance matrix $\Sigma$ of the rows of $X$. The latter can often be much easier to verify; for instance, it may be easily seen to hold as soon as $\Sigma$ is well-conditioned. 


\begin{proposition}\label{prop:res-eig}
     Let the rows of $X$ be i.i.d. $N(0, \Sigma)$. Suppose that $\Sigma^{1/2}$ satisfies $\RE_{\gamma}(s, \kappa_{\Sigma}(s))$ and let $\rho(\Sigma) := \max_{j} \sqrt{\Sigma_{jj}}$. Assume that
    \begin{equation}
        n \ge (36 \cdot 8)^2 \frac{(\rho(\Sigma))^2 s T_{\max} \log p}{(\kappa_{\Sigma}(s))^2}.
    \end{equation}
Then $X/\sqrt{n}$ satisfies $\RE_{\gamma}(s, \kappa_{\Sigma}(s) / 8)$ with probability at least $1 - C'\exp(-Cn)$ for some constants $C, C' > 0$.
\end{proposition}

\section{Random designs with a latent network geometry}

\subsection{Gaussian Graphical Models and Gaussian Free Fields}
\textit{Gaussian Free Fields} (abbrv.  GFF) have emerged as  important models of  correlated Gaussian fields,  that are naturally commensurate with the geometry of their ambient space.  In the case of graphs, the ambient geometry is spawned by the graph Laplacian.  These Gaussian processes also have important applications in physics, where they are of interest in the context of Euclidean quantum field theories \citep{friedli2017statistical}.

GFFs are in fact Gaussian Graphical Models (abbrv. GGM),  where the precision matrix of the Gaussian random field aligns with the Laplacian  of the graph,  thereby leading to a rich interaction between the statistical properties of the GGM and the Laplacian geometry of the underlying graph \citep{zhu2003semi,zhu2003combining,ma2013sigma,kelner2019learning,rasmussen2003gaussian}.  GGMs have emerged as   popular tools to model dependency relationships in data via a latent graph structure,  the choice of Gaussian randomness being often motivated by the fact that a Gaussian distribution maximises entropy within the constraints of a given covariance structure.  Applications of GGMs are ubiquitous, with use cases in diverse domains such as structural inference in biological networks, causal inference problems, speech recognition, and so on \citep{whittaker2009graphical,lauritzen1996graphical,edwards2012introduction,uhler2019gaussian}.  Since our approach  exploits in an essential manner the Laplacian geometry of the graph, the GFF is a natural GGM to examine within its ambit. 

For an in-depth introduction to the technical aspects of GFFs, we refer the reader to the excellent surveys \cite{sheffield2007gaussian} and \cite[Chap. 1]{berestycki2015introduction}; herein we will content ourselves with a brief description of its relevant features.   Broadly speaking,  a GFF is essentially a natural generalization of Brownian motion to general spaces,  with time replaced by,  e.g.,  the nodes of a network. 
We define the \textit{massive} GFF $\X_\t=(\X_\t(v))_{v \in V}$ on a graph $G=(V,E)$ with mass parameter $\t>0$ to be a mean-zero Gaussian field indexed by $V$ characterised by its \textit{precision matrix} (i.e., the inverse of the covariance matrix) given by $(L + \t I_{|V|})$, where $L$ is the (unnormalised) graph Laplacian on $G$ and $I_{|V|}$ is the $|V| \times |V|$ identity matrix (c.f. \cite{berestycki2015introduction}). The 
covariance matrix of $\X_\t$ is therefore given by $\S=(L + \t I_{|V|})^{-1}$. Note that $L$ itself is singular due to the all ones vector $\mathbbm{1}_V$ being in its kernel;   therefore we are aided by the strong convexity accorded by the mass parameter $\t>0$. 

\subsection{Stochastic Block Models and random designs}
In this section, we consider a very different model of graph clustering that is motivated by Stochastic Block Models (abbrv.,  SBM) that have attracted intense focus in statistics and machine learning applications in recent years  \citep{abbe2017community,goldenberg2010survey,holland1983stochastic,karrer2011stochastic}.  SBMs are  underpinned by a block matrix structure, indexed by the vertices of a graph,  with entries in $[0,1]$ that are constant on the blocks. These entries are the connection probabilities between the respective vertices.  In our setting,  it would be natural to consider the same block matrix pattern to generate the graph and the covariance structure.

For definiteness,  we divide the vertex set $V$ in $k$ groups $\{\calC_i\}_{i=1}^k$. We consider the $|V| \times |V|$  matrix $\cP$ whose rows and columns are indexed by the nodes of the graph $G$. For  $1 \le i,j \le k$,  let $\cP^{i,j}$ denote the submatrix indexed by the vertices in the groups $\calC_i$ (along the rows) and $\calC_j$ along the columns. For a vector $v \in \R^{|V|}$, we denote by $D(v)$ the diagonal matrix whose diagonal equals $v$. For brevity, we will denote by $\ind_i$ the indicator vector (in $\R^{|V|}$) of the nodes in $\calC_i$.  Let $a>b$ be numbers in $[0,1]$, such that $\cP^{i,i} = a (\ind_i \ind_i^\top -  D(\ind_i))$, and for $1 \le i \ne j \le k$, we have $\cP^{i,j} = b \ind_i \ind_j^\top$. Typically, $k$ is small compared $|V|$ and $a-b$ is taken to be large enough. The classical SBM is a random graph that is sampled with independent edges according to probabilities given by the matrix $\cP$.


In the  setting of our statistical problem, the graph $G$ underpinning the data may arise in two ways from such a matrix $\cP$. First, the underlying graph $G$ behind the data could simply be a realisation of a SBM driven by the matrix $\cP$, with the parameter $b$ being small so as to ensure the graph to be sparsely connected across blocks, and we have oracle access to $G$ (eg, possibly stemming from problem description).

Alternatively, we may consider a deterministic,  albeit weighted,  graph $G$,  whose edge weight between the vertices $u,v \in V$ equals $\cP_{uv}$.  The correlation structure of the covariates would then be given  by the covariance matrix $\S=I_{|V|} + \cP$ (assuming unit variances). A candidate network underpinning the data may then be constructed in a data-driven manner from such a block-structured sample correlation matrix $\S$ via thresholding. This approximating network would, in fact, be well-modelled by the SBM driven by the matrix $\cP$.

\subsection{Prediction guarantees and sample complexity bounds}  \label{sec:recovery}
In this section, we provide ballpark estimates on the quantities of interest in our random design models, that would be  enough to guarantee the effectiveness of our approach. We record here the main conclusions of our analysis, postponing the details to the appendix. 

For definiteness, we fix a polynomial decay of probability with which our recovery guarantees are to hold, which implies that the quantity $\eta=O(n^{-\a})$ in Theorem \ref{thm:prediction-consistency} for a fixed $\a>0$. Under this error tolerance, the appropriate choice of $\la$ may be shown to be $\la \gtrsim  \s_{\max}(\S) \sqrt{\frac{\log n}{n}}$, which in turn leads to  a prediction guarantee of 
\begin{align*} 
    \frac{1}{n}\|X(\hat{\beta}_{t, \lambda} - \beta^*)\|_2^2 \numberthis  \label{eq:PG} = O_P\l( \frac{\log n}{n}  \cdot \frac{ \s_{\max}(\S)}{\s_{\min}(\S)^2} \cdot  s |\calC_{\max}| + p e^{-t\lambda_g/2}  \r). 
\end{align*} 

Our goal here is to understand the order of the flow time $\tflow$ and the step count $\Ns$ (in terms of the other parameters of the problem) required to achieve a desired accuracy.  We will  bifurcate our analysis into two related parts.

\subsubsection{Bounds on \texorpdfstring{$\tflow$}{} for given \texorpdfstring{$n$}{}}
Given the data size $n$, we investigate the order of $t=\tflow$ at which the approximation error due to our heat flow based approach (roughly, the second term in \eqref{eq:PG}) becomes comparable to  the contribution to the prediction error bound for the classical group lasso methods that assume complete knowledge of the group structure (roughly, the first term in \eqref{eq:PG}). It may be shown that under very general circumstances, we have
\begin{equation} \label{eq:time_10}
    \tflow ~\gtrsim \frac{1}{\la_g} \log p + \frac{1}{\la_g} \log \l(  \frac{ n }{\log n} \cdot \frac{ \s_{\min}(\S)^2} { \s_{\max}(\S)} \cdot \frac{1} {s |\calC_{\max}|  }\r), 
\end{equation}
whereas for most models of interest, including the GFF and SBM based models in our purview, we may deduce the much simpler prescription $\tflow ~\gtrsim \frac{1}{\la_g} \max \{ \log p , \log n \}$. The details are provided in Section 3 of the appendix.

\subsubsection{Bounds on \texorpdfstring{$\tflow, n$}{} for target prediction guarantee \texorpdfstring{$\eps$}{}}

We fix a  threshold $\eps$, and make explicit prescriptions for the order of $n$ and $\tflow$ that will allow us to obtain a prediction error of order $O(\eps)$. To this end, we posit that the two terms on the right hand side of  \eqref{eq:PG} are separately $O(\eps^2)$. For $\tflow$, this entails that $p e^{-\la_g \tflow / 2} \lesssim \eps^2$, which translates into $\tflow \gtrsim \frac{1}{\la_g}\l( \log p + \log \frac{1}{\eps} \r)$.

For a prescription for $n$ given a target prediction error $\eps$, we have the bound
\begin{align*}
 \frac{n}{\log n}  \numberthis \label{eq:n_eps}  \gtrsim \max \l\{  \l( \frac{1}{\eps^2} \cdot s |\calC_{\max}| \cdot \frac{ \s_{\max}(\S)  }{ \s_{\min}(\S)^2}\r), \l( s |\calC_{\max}| \cdot \frac{\rho(\Sigma)^2  \log p}{\s_{\min}(\S)^2} \r) \r\}. 
\end{align*}
The details of the analysis are available in Section 3 of the appendix.

\subsection{Guarantees for typical clustered networks}
For typical clustered networks (for a concrete probabilistic model, see Section 3 in the appendix), we can further simplify above prescriptions on the heat flow time $\tflow$  to the thresholds 
\begin{equation} \label{eq:eq:time_well_clustered}
    \tflow ~\gtrsim \frac{1}{p} \cdot  \max \{ \log p , \log n \}; \quad \tflow ~\gtrsim \frac{1}{p} \cdot \max \{ \log p , \log \frac{1}{\eps} \},
\end{equation}
resp. for the settings where data size $n$ is given and where the target prediction error $\eps$ is given. These translate into the step count bounds
\begin{equation} \label{eq:eq:step_well_clustered}
    \Ns = O\l( \max \{ \log p , \log n \} \r); \; \Ns = O\l(\max \{ \log p , \log \frac{1}{\eps} \} \r).
\end{equation}
We refer the reader to Sec. 3 in the appendix for details. It may be noted that the flow time $\tflow$ and the step count $\Ns$ do not depend on the statistical properties of the covariates $X_i$, but depend only on the geometric properties of the network $G$.

\subsubsection{Explicit guarantees for GFF and block model designs}

In the setting of GFF and SBM based random designs, we may obtain further explicit guarantees on the sample complexity $n$, via application of techniques from spectral graph theory; the detailed analysis is provided in Section 3 of the appendix. 

For both SBM based  and  GFF based random designs on  well clustered networks, we have the following prescriptions that suffice with high probability:
\begin{equation} \label{eq:sample_cplx_SBM}
    \l( \frac{n}{\log n} \r)_{\text{SBM}}  \gtrsim \max \l\{ \frac{1}{\eps^2} \cdot \|\b^*\|_0,~~ \|\b^*\|_0 \cdot \log p \r\} 
\end{equation}
and
\begin{equation} \label{eq:sample_cplx_GFF}
    \l( \frac{n}{\log n} \r)_{\text{GFF}}  \gtrsim \max \l\{ \frac{1}{\eps^2} \cdot \|\b^*\|_0 \cdot  p,~~ \|\b^*\|_0 \cdot \log p \r\} 
\end{equation}

We observe that  the sample complexity bound \eqref{eq:sample_cplx_SBM}  for an SBM based design matches, upto logarithmic factors, the analogous bound for classical sparse reconstruction problems. For GFF based random designs,  it may be noted that the $\frac{1}{\eps^2} \cdot \|\b^*\|_0 \cdot  p$ term in \eqref{eq:sample_cplx_GFF}  comes from the first term in \eqref{eq:n_eps}, which, roughly speaking, reflects  the error incurred by classical group lasso. Thus, the linear dependence of the sample complexity on $p$ appears to be a fundamental characteristic of the problem for GFF based random designs, and is inherent to both classical group lasso and the present heat flow based methods. Empirical investigations also appear to corroborate this effect; for details we refer the reader to Section 4 in the appendix. Although  our method is still effective in a high dimensional setup (up to logarithmic factors),  devising  methodologies with greater efficiency for GFF based random designs, as well as theoretical investigations of information theoretic lower bounds in this setting, would be an interesting direction for future research.

\section{Experiments}\label{sec:exp}
\subsection{Simulations}
We take $n = 200, p = 100$ and $k = 4$ groups of relative sizes $(p_1, p_2, p_3, p_4)^\top / p = (0.16, 0.24, 0.40, 0.20)^\top$. We denote these groups by $\cC_i, i = 1, \ldots, k$.

We consider two models for $X$:
\begin{enumerate}
    \item \textbf{Gaussian with block diagonal covariance matrix.}
        The covariates $X \sim \mathcal{N}(0, \Sigma)$, with
        \begin{equation}\label{eq:block_diagonal_sigma}
            \Sigma = \begin{pmatrix}
                \Sigma_{p_1}(\rho_1) & \bzero & \bzero & \bzero \\
                \bzero & \Sigma_{p_2}(\rho_2) & \bzero & \bzero \\
                \bzero & \bzero & \Sigma_{p_3}(\rho_3) & \bzero \\
                \bzero & \bzero & \bzero & \Sigma_{p_4}(\rho_4)
            \end{pmatrix},
        \end{equation}
        where $\Sigma_d(\rho) = (1 - \rho) I_d + \rho \ind_d \ind_d^\top$ is the equi-correlation matrix of order $d$.

    \item \textbf{Gaussian free field.}
        We take a graph $G$ generated from a stochastic blockmodel on $p$ vertices with the groups $\cC_i, i = 1, \ldots, k$, and a connection probability of $a = 0.5$ within groups and $b = 0.025$ between groups. We then take a massive GFF on the graph $G$, with mass parameter set to the $(k + 1)$-th smallest eigenvalue of the unnormalised Laplacian corresponding to $G$.
\end{enumerate}

The true parameter $\beta$ is generated in the following way
\[
    \beta_i \begin{cases}
        \sim \mathrm{Uniform}(0.5, 0.7) & \text{ for } i \in \cC_1, \\
        = 0 & \text{ for } i \in \cC_2, \\
        \sim \mathrm{Uniform}(-0.7, -0.5) & \text{ for } i \in \cC_3, \\
        = 0 & \text{ for } i \in \cC_4.
    \end{cases}
\]

Finally, we generate the response from the linear model
\[
    Y = X\beta + \eps,
\]
where $\eps$ is a noise vector, independent of $X$, with isotropic covariance matrix $\sigma^2 I_n$.

We compare the subgradient and the stochastic block co-ordinate descent versions of our procedure (referred to as ``Heat Flow (SD)'' and ``Heat Flow (CD)'', respectively) against group lasso with group structure learned from spectral clustering \citep{von2007tutorial} on $\hat{L}$, an estimate of the Laplacian matrix obtained using Algorithm~\ref{alg:graph_estimation} with oracle knowledge of $k = 4$. To apply Algorithm~\ref{alg:graph_estimation}, an estimate $\hat{R}$ of the true correlation matrix is computed by normalising an estimate $\hat{\Sigma}$ of the true covariance matrix obtained using the \texttt{CovEst.2010RBLW} function from the R package \texttt{CovTools}, which implements the so-called \emph{Rao-Blackwell Ledoit-Wolf (LBRW)} estimator due to \cite{chen2010shrinkage}. For group lasso, the tuning parameter $\lambda$ is chosen by $5$-fold cross-validation (using the R package \texttt{gglasso}). The tuning parameters $\lambda$ and $t$ in both variants of our procedure are jointly tuned by $5$-fold cross-validation. We also report the performances of our estimators with the heat-flow dynamics run till $\tflow = 0.5 \cdot \min\{1, \hat{\lambda}_g^{-1}\}$, where $\hat{\lambda}_g$ is the $k$-th smallest eigenvalues of $\hat{L}$, and cross-validated over $\lambda$ only. We use ridge regression to initialise our algorithms. In Tables~\ref{table:sim-1} and \ref{table:sim-2}, we report the prediction error, estimation error, and two measures of support recovery, namely, sensitivity and specificity:
\begin{align*}
    \text{Sensitivity} &= \frac{\#\{i : \hat{\beta}_i \ne 0, \beta^*_i \ne 0\}}{\#\{i : \beta^*_i \ne 0\}}, \\
    \text{Specificity} &= \frac{\#\{i : \hat{\beta}_i = 0, \beta^*_i = 0\}}{\#\{i : \beta^*_i = 0\}}.
\end{align*}

In the first experiment, we see that our proposed method (the SD variant) based on the heat flow penalty has comparable performance to group lasso in terms of prediction/estimation error, without requiring an explicit knowledge of the group structure, or even of the number of groups.  
In terms of support recovery, the proposed method appears to be much superior. 
The estimated $\beta$ for different methods in the experiment with block diagonal covariate structure is shown in Figure~\ref{fig:sim-1}. This experiment also highlights that even in a favourable setup, group lasso may suffer from low specificity.

\begin{figure}[!th]
    \centering
    \begin{tabular}{cc}
        \includegraphics[width = 0.4\textwidth]{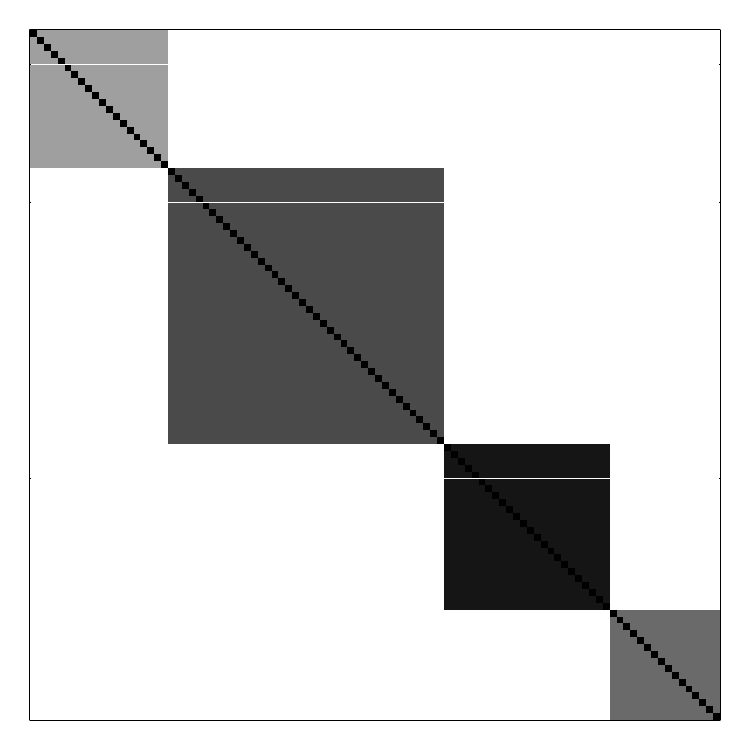} &
        \includegraphics[width = 0.4\textwidth]{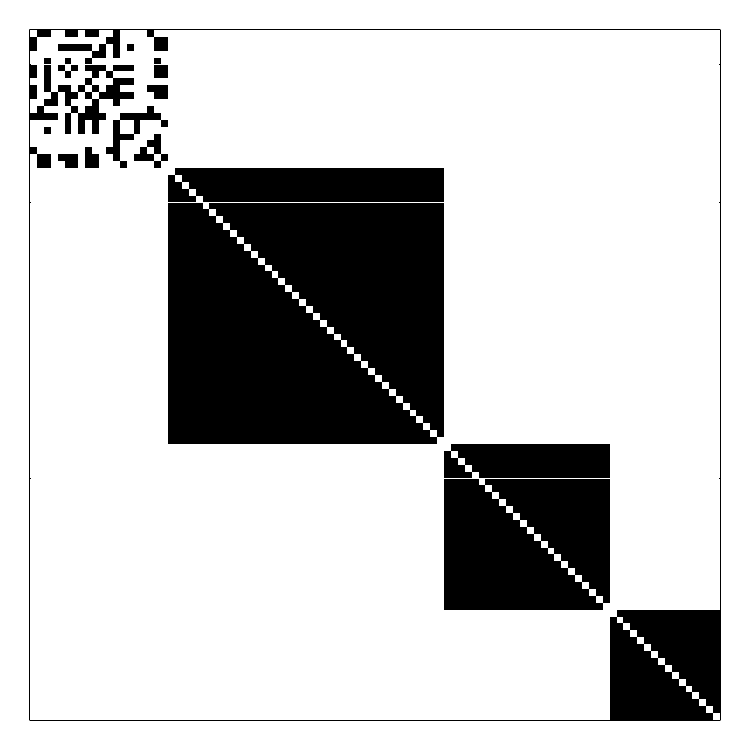} \\
        (a) & (b) \\
    \end{tabular}
    \begin{tabular}{c}
        \includegraphics[width = 0.7\textwidth]{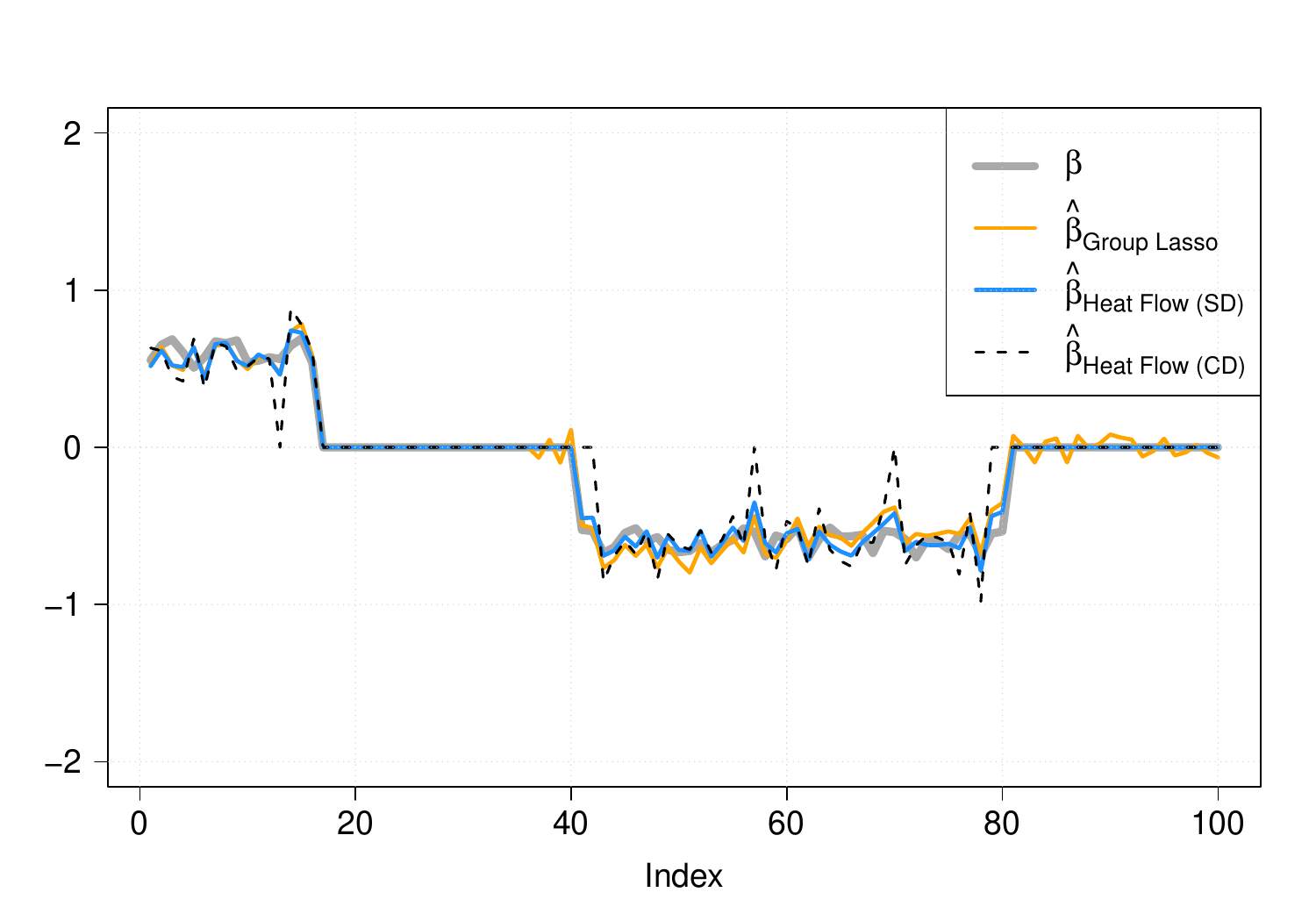} \\
        (c)
    \end{tabular}
    \caption{One of Monte Carlo runs in the simulation experiment with a block diagonal covariance structure for $X$: (a) true covariance matrix; (b) estimated network; (c) estimated $\beta$ for group lasso and Heat Flow (SD).}
    \label{fig:sim-1}
\end{figure}

\begin{table}[!t]
    \footnotesize
    \centering
    \begin{tabular}{c|c|c|c|c|c|}
        \toprule
        \makecell{} & \makecell{Group Lasso} & \makecell{Heat Flow (SD) \\ till $\tflow$} & \makecell{Heat Flow (SD) \\ Full CV} & \makecell{Heat Flow (CD) \\ till $\tflow$} & \makecell{Heat Flow (CD) \\ Full CV} \\
        \midrule
        Prediction error & 0.03 (0.00) & 0.03 (0.01) & 0.03 (0.00) & 0.11 (0.05) & 0.11 (0.04) \\
        Estimation error & 0.90 (0.20) & 0.65 (0.19) & 0.62 (0.08) & 1.72 (0.32) & 1.68 (0.30) \\
        Sensitivity      & 1.00 (0.00) & 1.00 (0.01) & 1.00 (0.00) & 0.91 (0.06) & 0.92 (0.05) \\
        Specificity      & 0.33 (0.28) & 1.00 (0.00) & 1.00 (0.00) & 0.94 (0.03) & 0.94 (0.03) \\
        \bottomrule
    \end{tabular}
    \caption{Comparison when $X \sim \mathcal{N}(0, \Sigma)$ with a block diagonal $\Sigma$ as in \eqref{eq:block_diagonal_sigma}. We take $(\rho_1, \rho_2, \rho_3, \rho_4)^\top = (0.6, 0.9, 0.7, 0.4)^\top$. Reported are the average scores are based on $50$ Monte Carlo runs (the numbers inside parentheses are the corresponding standard errors). The average accuracy of clustering in this setting before applying group lasso is $0.994 \, (0.027)$.}
    \label{table:sim-1}
\end{table}

The second experiment is much harder, the average clustering accuracy being $0.500 \, (0.096)$. This experiment highlights the advantage of not clustering when the group structure in the available network is weak. The proposed methods are seen to outperform group lasso.

The two experiments also suggest that running heat flow upto $\tflow$ is quite competitive with a full cross-validation over a range of $t$'s.

\begin{figure}[!th]
    \centering
    \begin{tabular}{cc}
        \includegraphics[width = 0.4\textwidth]{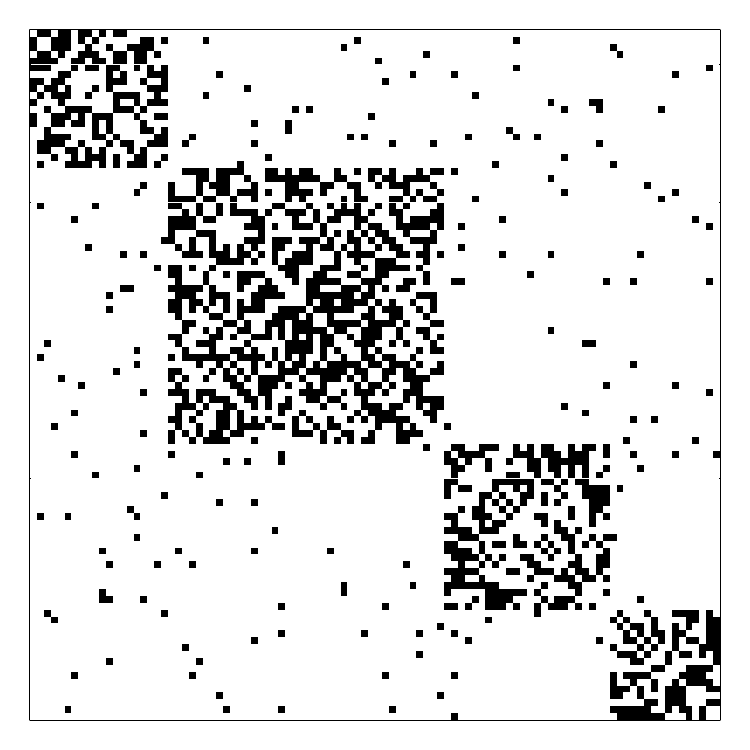} &
        \includegraphics[width = 0.4\textwidth]{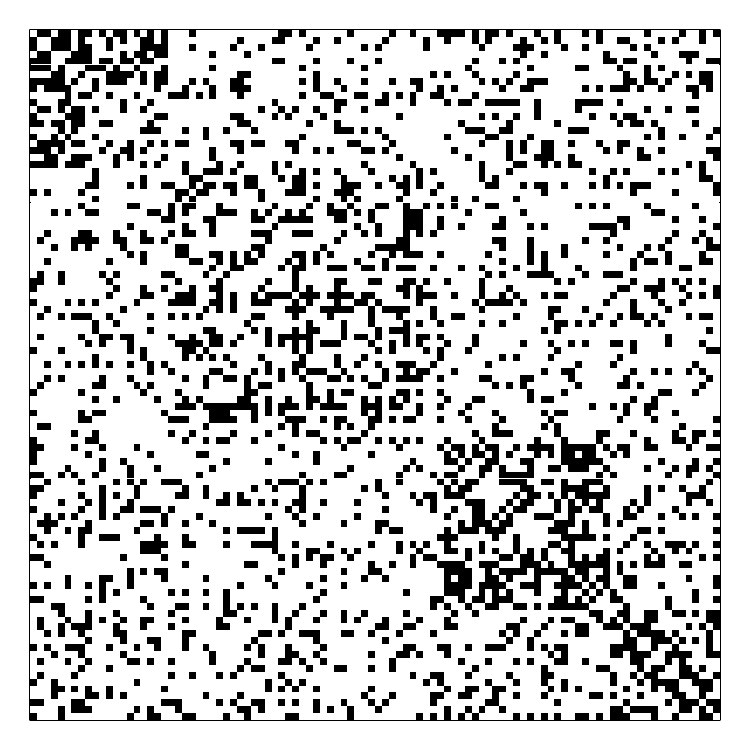} \\
        (a) & (b) \\
    \end{tabular}
    \begin{tabular}{c}
        \includegraphics[width = 0.7\textwidth]{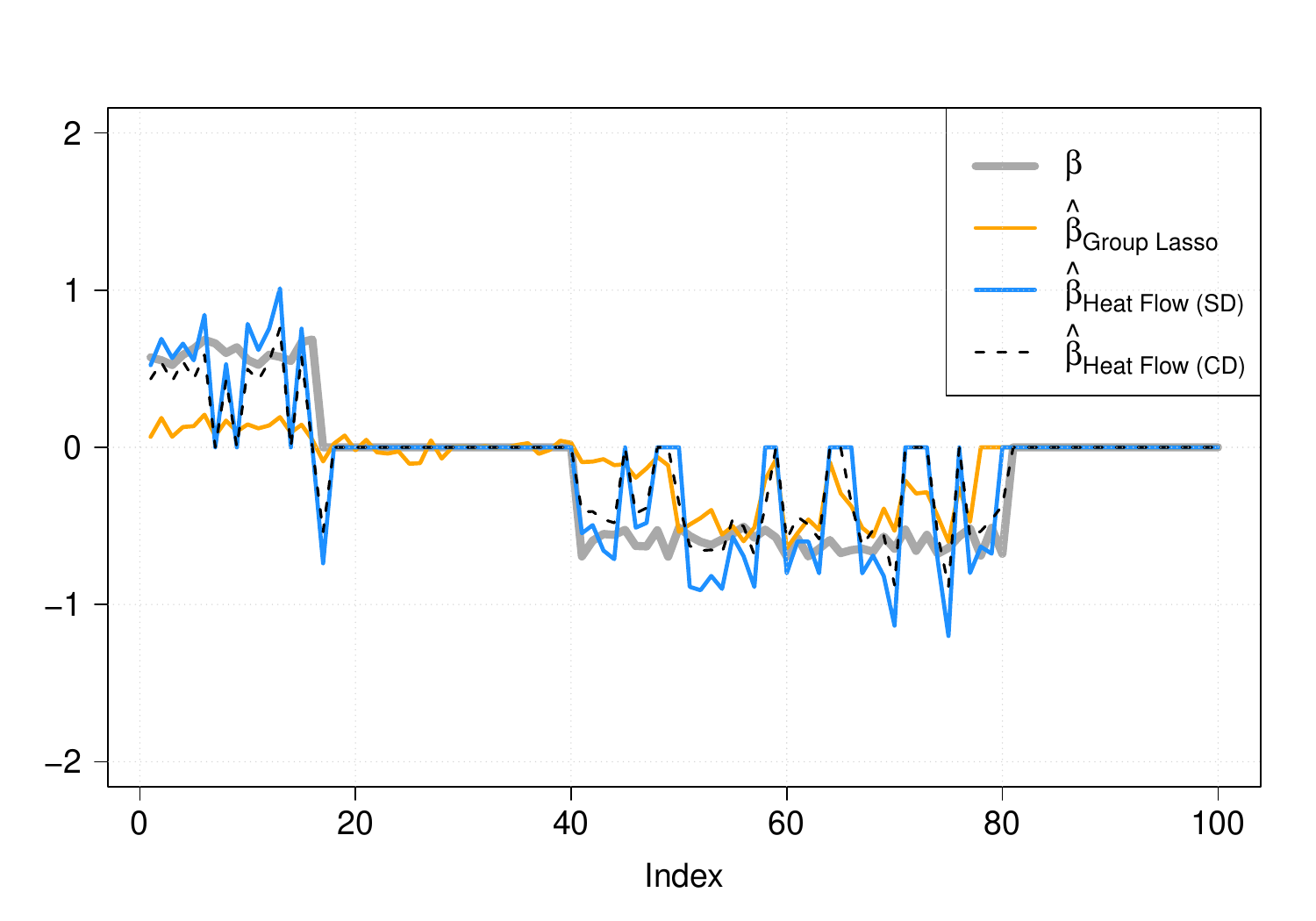} \\
        (c)
    \end{tabular}
    \caption{One of Monte Carlo runs in the simulation experiment with the GFF model for $X$: (a) true network; (b) estimated network; (c) estimated $\beta$ for Group Lasso and Heat Flow (SD).}
    \label{fig:sim-2}
\end{figure}
\begin{table}[!t]
    \footnotesize
    \centering
    \begin{tabular}{c|c|c|c|c|c}
        \toprule
        \makecell{} & \makecell{Group Lasso} & \makecell{Heat Flow (SD) \\ till $\tflow$} & \makecell{Heat Flow (SD) \\ Full CV} & \makecell{Heat Flow (CD) \\ till $\tflow$} & \makecell{Heat Flow (CD) \\ Full CV} \\
        \midrule
        Prediction error & 0.07 (0.01) & 0.06 (0.01) & 0.06 (0.01) & 0.06 (0.01) & 0.06 (0.01) \\
        Estimation error & 3.69 (0.81) & 3.27 (0.23) & 3.23 (0.21) & 3.16 (0.20) & 3.12 (0.24) \\
        Sensitivity      & 0.93 (0.14) & 0.61 (0.07) & 0.60 (0.06) & 0.59 (0.07) & 0.61 (0.08) \\
        Specificity      & 0.34 (0.33) & 0.91 (0.06) & 0.92 (0.04) & 0.92 (0.04) & 0.91 (0.06) \\
        \bottomrule
    \end{tabular}
    \caption{Comparison when $X$ is drawn from a Gaussian free field. Reported are the average scores are based on $50$ Monte Carlo runs (the numbers inside parentheses are the corresponding standard errors). The average accuracy of clustering in this setting before applying group lasso is $0.500 \, (0.096)$.}
    \label{table:sim-2}
\end{table}

\subsection{Real data}

In this section, we compare our methods against group lasso in terms of test-set performance in four real-world data sets. For group lasso, the group structure is learned by applying spectral clustering on $\hat{L}$, the estimated Laplacian matrix obtained by applying Algorithm~\ref{alg:graph_estimation}, with $k$ set to be the number of eigenvalues of $\hat{L}$ less than $0.01$. For each data set, we use an 80:20 split into training and test sets. The test-set errors are reported in Table~\ref{table:real_data}. We observe comparable performance in all four data sets.

\textbf{Email spam data.}
We consider the well-known spambase data set\footnote{https://archive.ics.uci.edu/ml/datasets/spambase} containing 4601 emails classified as spam/non-spam. There are 57 explanatory variables. We fit a logistic regression model with group lasso and heat flow penalties. We report the test-set misclassification error in the second column of Table~\ref{table:real_data}.

\textbf{Gene-expression data.}
We consider gene expression data from the microarray experiments of mammalian eye tissue samples of \cite{scheetz2006regulation}. The response variable is the expression level of the TRIM32 gene. There are 200 predictor variables corresponding to different gene probes. The sample size is 120. We report the test-set mean-squared error (MSE) in the third column of Table~\ref{table:real_data}.

\textbf{Climatological data.}
From the NCEP/NCAR reanalysis data set, we took the monthly average temperature of the Delhi-NCR region as the response variable. We took monthly average temperature, pressure, precipitation, wind-speed, etc. of $2.5^\circ \times 2.5^\circ$ blocks on the Bay of Bengal and the Arabian Sea as covariates. In total there were 101 such blocks, giving us $p = 606$ explanatory variables. We have these measurements for $n = 886$ months, starting January, 1947 till October 2021. We first removed seasonal variations and fitted a linear trend afterwards as preprocessing steps (as described in \cite{chatterjee2012sparse}). We report the test-set MSE in the fourth column of Table~\ref{table:real_data}.

\textbf{Stock-market data.}
We have data on daily highs of the NIFTY 50 index from the National Stock Exchange (NSE) of India for 49 companies for $n = 2598$ days staring from November 4, 2010 till April 30, 2021. We use the mean index of $9$ companies in the financial sector as our response variable, and use the indices of the rest of the $p = 40$ companies as covariates. We report the test-set MSE in the fifth column of Table~\ref{table:real_data}.

\begin{table}[!htbp]
    \centering
    \begin{tabular}{c|c|c|c|c}
        \toprule
        Method         & spam   & gene   & climate & stock  \\
        \midrule
        Group lasso    & 0.18 & 0.15 & 0.06  & 0.02 \\
        Heat flow (SD) & 0.11 & 0.15 & 0.07  & 0.02 \\
        Heat flow (CD) & 0.11 & 0.15 & 0.07  & 0.03 \\
        \bottomrule
    \end{tabular}
    \caption{Test-set error on real data.}
    \label{table:real_data}
\end{table}

\section{Proofs} \label{sec:theory}

\begin{table}[!t]
    \centering
    \caption{Glossary of notation}
    \label{table:glossary}
    \begin{tabular}{c|l}
    Notaion & Meaning \\ \hline\hline
    $n$ &   Number of observations \\
    $p$ &   Number of variables \\
    $X$ &   Design matrix $n \times p$ \\
    $y$ &   Response variable \\ \hline
    $k$ &   Number of groups of covariates \\
    $\cC_i$ & The $i$-th group \\
    $T_i$ & Size of $\cC_{i}$ \\
    $\cC_{\max}$ & The largest group \\
    $T_{\max}$ & Size of $\cC_{\max}$ \\
    $L$ & Laplacian matrix of the underlying graph \\
    $\lambda_g$ & $(k + 1)$-th smallest eigenvalue of $L$ \\ \hline
    $\beta^*$ & True signal vector \\
    $A(\beta)$ & Active set of $\beta$: $\{j : \|\beta^j\|_2 \ne 0\}$ \\
    $I(\beta)$ & Active indices of $\beta$: $\cup_{j \in A(\beta)} \cC_j$ \\
    $s$ & Group sparsity level ($s \ge |A(\beta^*)|$) \\
    $\kappa$ &  Restricted eigenvalue constant of $X/\sqrt{n}$ \\
    $\tflow$ & Time till heat flow is run \\
    $H$ & Heat flow matrix (containing endpoints of heat flows) \\
    \hline\hline
    \end{tabular}
\end{table}

In what follows, in the setting of groups / clusters $\{\calC_i\}_{i=1}^k$ we will interchangeably use the notations $T_j=|\calC_j|$ and $T_{\max}=|\calC_{\max}|=\max \{ |\calC_i| : 1\le i \le k\}$. For the convenience of the reader, a glossary of notations is given in Table~\ref{table:glossary}. We begin with some auxiliary results whose proofs are deferred till the end of this section.

\subsection{Some auxiliary results} \label{sec:results}
We begin with a key technical result, which compares the heat-flow penalty $\Lambda_t(\beta)$ with the group lasso penalty $\Lambda_{\infty}(\beta)$. It makes precise the intuition described in Section~\ref{sec:intuition}.
\begin{lemma}\label{lem:pen-approx}
    For all $t$ such that
    \[
        \Xi(\beta; t, G) := e^{-t\lambda_g} \|\beta \odot \beta\|_2 \le \frac{1}{2}\min_{i \in A(\beta)}\frac{\|\beta_{\cC_i}\|_2^2}{T_i},
    \]
    one has
    \begin{equation}\label{eq:comp-to-glasso-bd}
        |\Lambda_t(\beta) - \Lambda_{\infty}(\beta)| \le p \sqrt{\Xi(\beta; t, G)}.
    \end{equation}
\end{lemma}
For $B[\beta^*; \epsilon] := \{ \beta \mid n^{-1/2}\|X(\beta - \beta^*)\|_2 \le \epsilon\}$, define
\begin{equation}\label{eq:gammadef}
    \Gamma(\beta^*; \epsilon) := \min_{\beta \in B[\beta^*; \epsilon]}  \frac{1}{2\|\beta \odot \beta\|_2}\min_{i \in A(\beta)}\frac{\|\beta_{\cC_i}\|_2^2}{T_i}.
\end{equation}
\begin{corollary}\label{thm:closeness-of-losses}
    Let $\epsilon_0 > 0$ be such that $\Gamma(\beta^*; \epsilon_0) > 0$.
    Then for all large enough $t$ such that $(p - k) e^{-t \lambda_g} \le \Gamma(\beta^*; \epsilon_0)$, we have
    \[
        \max_{\beta \in B[\beta^*; \epsilon_0]} |\Lambda_t(\beta) - \Lambda_{\infty}(\beta)| \le C_{\beta^*, \epsilon_0} p 
        e^{-t\lambda_g/2},
    \]
    where $C_{\beta^*, \epsilon_0} = \max_{\beta \in B[\beta^*; \epsilon_0]} \|\beta \odot \beta\|_2$.
\end{corollary}

We set 
\begin{equation} \label{eq:lambda_bounds}
    \Psi_j = \frac{1}{n}X_j^\top X_j  \quad \text{and} \quad  \Lambda(X; \eta) = \max_{1 \le j \le k} \sqrt{\|\Psi_j\|_{\op}}\bigg(1 +  \sqrt{\frac{4\log \eta^{-1}}{T_j}}\bigg).
\end{equation}
Then we have: 
\begin{lemma}\label{lem:approx-min}
    Let $\lambda \ge \frac{\sigma}{\sqrt{n}} \Lambda(X; \eta)$, and, for $0 < \epsilon \le \epsilon_0$, let $\hat{\beta}_{t, \lambda} = \hat{\beta}_{t, \lambda}^{(\epsilon)}$ be the minimiser of $F_{t, \lambda}(\beta)$ in $B[\beta^*, \epsilon]$. Assume also that $8 T_{\max} \lambda \|\beta^*\|_{2, 1} \le \epsilon$. Then, with probability at least $1 - 2k\eta$, we have that $\hat{\beta}_{t, \lambda}$ is an approximate minimiser of the group Lasso objective $F_{\infty, \lambda}$ in the sense that
    \[
        F_{\infty, \lambda}(\hat{\beta}_{t, \lambda}) \le \min_{\beta} F_{\infty, \lambda}(\beta) + 2 C_{\beta^*, \epsilon_0} p 
        e^{-t\lambda_g/2}.
    \]
\end{lemma}
Using Lemma~\ref{lem:approx-min}, we can prove an approximate sparsity oracle inequality for $\hat{\beta}_{t, \lambda}$.
\begin{lemma}(Approximate Sparsity Oracle Inequality)\label{lem:approx-oracle}
    Under the assumptions of Lemma~\ref{lem:approx-min}, we have with probability at least $1 - 2k\eta$ that for any $\beta \in \R^p$,
    \begin{align}\label{eq:approx-sparsity-oracle} \nonumber
        \frac{1}{2n}\|X(\hat{\beta}_{t, \lambda} - \beta^*)\|_2^2 &+ \frac{\lambda}{2} \sum_j \sqrt{T_j} \|\hat{\beta}_{t, \lambda}^j - \beta^j\|_2 \\
                                                                  &\le \frac{1}{2n} \|X(\beta - \beta^*)\|_2^2 + 2 \lambda \sum_{j \in A(\beta)} \sqrt{T_j} \min \{\|\beta^j\|_2 \|\hat{\beta}_{t, \lambda}^j - \beta^j\|_2\} + E,
    \end{align}
    where $E = 2 C_{\beta^*, \epsilon_0} p 
    e^{-t\lambda_g/2}$.
\end{lemma}
The approximate sparsity oracle inequality will give us a prediction consistency result. Without any further assumptions we have a slow-rate result. For getting faster rates we assume a restricted eigenvalue property.

\subsection{Proofs of the main results}\label{sec:proof}

\begin{proof}[Proof of Theorem~\ref{thm:prediction-consistency}]
    The slow-rate prediction consistency result in \eqref{eq:PC_slow} follows immediately from \eqref{eq:approx-sparsity-oracle} by taking $\beta = \beta^*$. 

    We now prove the fast-rate result in \eqref{eq:PC_fast}.
Let $\Delta_t = \hat{\beta}_{t, \lambda} - \beta^*$. From the approximate sparsity oracle inequality \eqref{eq:approx-sparsity-oracle} with $\beta = \beta^*$, we have 
\begin{equation}\label{eq:approx-sparsity-implication-1}
    \frac{1}{2n} \|X\Delta_t\|_2^2 \le 2 \lambda \sum_{j \in A(\beta^*)} \sqrt{T_j} \|\Delta_t^j\|_2 + E, 
\end{equation}
where $A(\beta^*) = \{j : \beta^*_{\cC_j} \ne 0\}$. Also, from \eqref{eq:approx-sparsity-oracle}, one has
\[
    \frac{\lambda}{2} \sum_j \sqrt{T_j}\|\Delta_t^j\|_2 \le 2 \lambda \sum_{j \in A(\beta^*)} \sqrt{T_j} \|\Delta_t^j\|_2 + E
\]
which implies
\begin{equation}\label{eq:bound-on-inactive-set-from-approx-sparsity-oracle}
    \lambda \sum_{j \in A(\beta^*)^c} \sqrt{T_j} \|\Delta_t^j\|_2 \le 3 \lambda \sum_{j \in A(\beta^*)^c} \sqrt{T_j} \|\Delta_t^j\|_2 + 2E.
\end{equation}


%
\noindent
\textbf{Case I.} Suppose first that 
\begin{equation}\label{eq:case-1}
    2E \le \varpi \lambda \sum_{j \in A(\beta^*)} \sqrt{T_j} \|\Delta_t^j\|_2.
\end{equation}
Then \eqref{eq:bound-on-inactive-set-from-approx-sparsity-oracle} gives
\[
    \lambda \sum_{j \in A(\beta^*)^c} \sqrt{T_j} \|\Delta_t^j\|_2 \le (3 + \varpi) \lambda \sum_{j \in A(\beta^*)} \sqrt{T_j} \|\Delta_t^j\|_2,
\]
i.e.
\[
    \sum_{j \in A(\beta^*)^c} \sqrt{T_j} \|\Delta_t^j\|_2 \le (3 + \varpi) \sum_{j \in A(\beta^*)} \sqrt{T_j} \|\Delta_t^j\|_2,
\]
The assumption of $\RE_{3 + \varpi}(s, \kappa)$ then implies that (since $s \ge |A(\beta^*)|$)
\begin{equation}\label{eq:RE-implication}
    \|(\Delta_t)_A\|_2 \le \frac{\|X \Delta_t\|}{\kappa \sqrt{n}}.
\end{equation}
Now plugging the bound \eqref{eq:case-1} into \eqref{eq:approx-sparsity-implication-1} we obtain
\begin{align*}
    \frac{1}{2n} \|X\Delta_t\|_2^2 &\le \frac{(4 + \varpi)}{2} \lambda \sum_{j \in A(\beta^*)} \sqrt{T_j} \|\Delta_t^j\|_2 \\
                                   &\le \frac{(4 + \varpi) \lambda}{2} \bigg(\sum_{j \in A(\beta^*)} T_j\bigg)^{1/2} \|(\Delta_t)_A\|_2 \qquad(\text{Cauchy-Schwarz})\\
                                   &\le \frac{(4 + \varpi) \lambda}{2} \bigg(\sum_{j \in A(\beta^*)} T_j\bigg)^{1/2} \frac{\|X \Delta_t\|}{\kappa\sqrt{n}}, \qquad (\text{using \eqref{eq:RE-implication}})
\end{align*}
which yields
\begin{equation}\label{eq:pred-error-bound-case-1}
    \frac{\|X\Delta_t\|_2}{\sqrt{n}} \le \frac{(4 + \varpi) \lambda}{\kappa} \bigg(\sum_{j \in A(\beta^*)} T_j\bigg)^{1/2}.
\end{equation}
Again, from the approximate sparsity oracle inequality \eqref{eq:approx-sparsity-oracle} with $\beta = \beta^*$, we have
\begin{align*}
    \lambda \sum_j \sqrt{T_j}\|\Delta_t^j\|_2 &\le 4 \lambda \sum_{j \in A(\beta^*)} \sqrt{T_j} \|\Delta_t^j\|_2 + 2E \\ 
                                              &\le (4 + \varpi) \lambda \sum_{j \in A(\beta^*)} \sqrt{T_j} \|\Delta_t^j\|_2 \\
                                              &\le (4 + \varpi) \lambda \bigg(\sum_{j \in A(\beta^*)} T_j\bigg)^{1/2} \|(\Delta_t)_A\|_2 \\
                                              &\le (4 + \varpi) \lambda \bigg(\sum_{j \in A(\beta^*)} T_j\bigg)^{1/2} \frac{\|X \Delta_t\|_2}{\sqrt{n} \kappa} \qquad (\text{using \eqref{eq:RE-implication}})\\
                                              &\le \frac{(4 + \varpi)^2 \lambda^2}{\kappa^2} \sum_{j \in A(\beta^*)} T_j. \qquad(\text{using \eqref{eq:pred-error-bound-case-1}})
\end{align*}
Hence
\begin{equation}\label{eq:est-error-bound-case-1}
    \|\Delta_t\|_{2, 1} = \sum_{j} \|\Delta_t^j\|_2 \le \frac{1}{\sqrt{T_{\min}}}\sum_{j} \sqrt{T_j} \|\Delta_t^j\|_2 \le \frac{1}{\sqrt{T_{\min}}} \frac{(4 + \varpi)^2 \lambda}{\kappa^2} \sum_{j \in A(\beta^*)} T_j.
\end{equation}

\noindent
\textbf{Case II.} Now consider the complementary event 
\[
    2E > \varpi \lambda \sum_{j \in A} \sqrt{T_j} \|\Delta_t^j\|_2.
\]
Via \eqref{eq:approx-sparsity-implication-1}, this implies that
\begin{equation}\label{eq:pred-error-bound-case-2}
    \frac{1}{n} \|X\Delta_t\|_2^2 \le 4 \lambda \sum_{j \in A(\beta^*)} \sqrt{T_j} \|\Delta_t^j\|_2 + 2E \le \bigg(\frac{8}{\varpi} + 2\bigg) E.
\end{equation}
Also,
\[
    \lambda \sum_j \sqrt{T_j} \|\Delta_t^j\| \le 4 \lambda \sum_{j \in A} \sqrt{T_j} \|\Delta_t^j\| + 2E \le \frac{8E}{\varpi} + 2E = \bigg(\frac{8}{\varpi} + 2\bigg) E,
\]
i.e.
\begin{equation}\label{eq:est-error-bound-case-2}
    \|\Delta_t\|_{2, 1} \le \frac{1}{\sqrt{T_{\min}}} \frac{1}{\lambda} \bigg(\frac{8}{\varpi} + 2\bigg) E.
\end{equation}

Combining the prediction error bounds in the two cases, i.e. \eqref{eq:pred-error-bound-case-1} and \eqref{eq:pred-error-bound-case-2}, we get
\[
    \frac{1}{n}\|X(\hat{\beta}_{t, \lambda} - \beta^*)\|_2^2 \le  \frac{(4 + \varpi)^2 \lambda^2}{\kappa^2} s T_{\max} + \bigg(\frac{8}{\varpi} + 2\bigg) E,
\]
from which the fast rate result in \eqref{eq:PC_fast} follows.

Similarly, combining the estimation error bounds the in the two cases, i.e. \eqref{eq:est-error-bound-case-1} and \eqref{eq:est-error-bound-case-2}, we obtain
\[
    \|\hat{\beta}_{t, \lambda} - \beta^*\|_{2, 1} \le \frac{1}{\sqrt{T_{\min}}} \frac{(4 + \varpi)^2 \lambda}{\kappa^2} \sum_{j \in A(\beta^*)} T_j + \frac{1}{\sqrt{T_{\min}}} \frac{1}{\lambda}\bigg(\frac{8}{\varpi} + 2\bigg) E.
\]
The stated bound \eqref{eq:estimation-error} is a simple consequence of this.
\bk
\end{proof}

For proving Theorem~\ref{thm:support-recovery}, we need the following lemma whose proof is an adaptation of the proof of a similar bound derived in Theorem 5.1 of \cite{lounici2011oracle}.
\begin{lemma}\label{lem:2-infty-bound}
    Under the assumptions of Theorem~\ref{thm:prediction-consistency}-(b), we have with high probability that
    \[
        \|\Delta_t\|_{2, \infty} \le \frac{c}{\phi} \lambda \sqrt{T_{\max}}.
    \]
\end{lemma}
\begin{proof}[Proof of Theorem~\ref{thm:support-recovery}]
From Lemma~\ref{lem:2-infty-bound}, for $j \in A(\beta^*)$,
\[
    \|(\hat{\beta}_{\lambda, t})^j - (\beta^*)^j\|_{\infty} \le \|(\hat{\beta}_{\lambda, t})^j - (\beta^*)^j\|_{2} \le \frac{c}{\phi} \lambda \sqrt{T_{\max}}.
\]
So for any $i \in \cC_j$, one has
\[
    |\beta^*_i| - \frac{c}{\phi}\lambda \sqrt{T_{\max}} \le |(\hat{\beta}_{\lambda, t})_i|.
\]
Thus
\begin{align*}
    \min_{j \in A(\beta^*)} \min_{i \in \cC_j} |(\hat{\beta}_{\lambda, t})_i| &\ge \min_{j \in A(\beta^*)} \min_{i \in \cC_j} |\beta^*_i| - \frac{c}{\phi}\lambda \sqrt{T_{\max}} \\
        &> \frac{c}{\phi}\lambda \sqrt{T_{\max}}. \qquad\text{(by \eqref{eq:min_beta})}
\end{align*}
On the other hand, by Lemma~\ref{lem:2-infty-bound} again,
\[
    \max_{j \notin A(\beta^*)} \max_{i \in \cC_j}|(\hat{\beta}_{t, \lambda})_i| \le \max_{j \notin A(\beta^*)}\|(\hat{\beta}_{\lambda, t})^j\|_2 \le \frac{c}{\phi} \lambda \sqrt{T_{\max}}.
\]
This completes the proof.
\end{proof}
\bk


\begin{proof}[Proof of Proposition~\ref{prop:res-eig}] 
    For any $\Delta \in \R^p$, by Cauchy-Schwarz, we have
    \[
        \sum_{i \in \cC_j} |\Delta_i| \le \sqrt{T_j} \|\Delta^j\|.
    \]
    Thus if $\Delta$ and $A$ are as in the definition of $\RE_{\gamma}(s)$, then
    \begin{align*}
        \sum_{j \notin A}\sum_{i \in \cC_j} |\Delta_i| &\le \sum_{j \notin A} \sqrt{T_j} \|\Delta^j\| \le \gamma \sum_{j \in A} \sqrt{T_j} \|\Delta^j\|.
    \end{align*}
    Hence
    \[
        \|\Delta\|_1 \le (1 + \gamma) \sum_{j \in A} \sqrt{T_j} \|\Delta^j\|.
    \]
    Another application of Cauchy-Schwarz gives
    \[
        \|\Delta\|_1 \le (1 + \gamma) \bigg(\sum_{j \in A} T_j\bigg)^{1/2} \|\Delta_A\| \le (1 + \gamma) \sqrt{s T_{\max}} \|\Delta_A\|.
    \]
    Now using Theorem 1 of \cite{raskutti10a} we get that with probability at least $1 - C'\exp(-Cn)$,
    \begin{align*}
        \frac{\|X \Delta\|}{\sqrt{n}} &\ge \frac{1}{4}\|\Sigma^{1/2}\Delta\| - 9 \rho(\Sigma) \sqrt{\frac{\log p}{n}}\|\Delta\|_1 \\
                                      &\ge \bigg(\frac{\kappa_{\Sigma}(s)}{4} - 9(1 + \gamma)\rho(\Sigma) \sqrt{\frac{s T_{\max}\log p}{n}}\bigg)\|\Delta_A\| \\
                                      &\ge \frac{\kappa_{\Sigma}(s)}{8} \|\Delta_A\|,
    \end{align*}
    provided
    \[
        n \ge (9 (1 + \gamma) \cdot 8)^2 \frac{(\rho(\Sigma))^2 s T_{\max} \log p}{(\kappa_{\Sigma}(s))^2}.
    \]
    This completes the proof.
\end{proof}
\bk
\subsection{Proofs of the auxiliary results}
\begin{proof}[Proof of Lemma~\ref{lem:pen-approx}]
    We begin with the decomposition
    \[
        e^{-tL}(\beta \odot \beta) = \sum_i e^{-t\lambda_i} \langle v_i, \beta \odot \beta \rangle v_i. 
    \]
    For the eigenspace corresponding to $0$, we choose the basis $\{\mathbf{1}_{\cC_i} / \sqrt{|\cC_i|}\}_{i = 1}^k$. Thus we may write
    \[
        e^{-tL}(\beta \odot \beta) = g + \xi,
    \]
    where
    \[
        g = \sum_{i = 1}^k \frac{\|\beta_{\cC_i}\|_2^2}{|\cC_i|}\mathbf{1}_{\cC_i} \quad \text{ and } \quad \xi = \sum_{i > k} e^{-t\lambda_i} \langle v_i, \beta \odot \beta \rangle v_i.
    \]
    Recall that $\Lambda_t(\beta) = \Phi(e^{-tL}(\beta \odot \beta), \bone)$, where $\Phi(\beta) = (\sqrt{|\beta_1|}, \ldots, \sqrt{|\beta_p|})^\top$. Also, note that $\Lambda_{\infty}(\beta) = \langle \Phi(g), \mathbf{1} \rangle$. Therefore
    \begin{align*}
        |\Lambda_t(\beta) & - \Lambda_{\infty}(\beta)|                          \\
                          & = |\langle \Phi(g + \xi) - \Phi(g), \mathbf{1}\rangle | \\
                          & \le \sum_{\ell = 1}^p |\sqrt{|g_\ell + \xi_\ell|} - \sqrt{|g_\ell|}|  \\
                          & \le \sum_{\ell \in I(\beta)} \frac{|\xi_\ell|}{2\sqrt{\theta_\ell}} + \sum_{\ell \notin I(\beta)} \sqrt{|\xi_\ell|} \quad (\text{where }\theta_\ell = \alpha_\ell |g_\ell + \xi_\ell| + (1 - \alpha_\ell) |g_\ell| \text{ for some } \alpha_\ell \in [0, 1])
    \end{align*}
    Now
    \[
        \xi_\ell = \sum_{i > k} e^{-t\lambda_i} \langle v_i, \beta \odot \beta \rangle v_{i, \ell}.
    \]
    Notice that
    \begin{align*}
        |\xi_\ell| &\le  e^{-t \lambda_g} \sum_{i > k} |\langle v_i, \beta \odot \beta \rangle| |v_{i, \ell}| \\
                   &\le e^{-t\lambda_g} \bigg(\sum_{i > k} |\langle v_i, \beta \odot \beta \rangle|^2\bigg)^{1/2} \bigg(\sum_{i > k} |v_{i, \ell}|^2\bigg)^{1/2} \\ 
                   &\le e^{-t\lambda_g} \|\beta \odot \beta\|_2 \\
                   &= \Xi(\beta; t, G),
    \end{align*}
    where in the penultimate line we have used the fact the eigenvectors $v_i$ form an orthonormal basis of $\R^p$. Indeed, then $\sum_{i > k} |\langle v_i, \beta \odot \beta \rangle|^2 \le \|\beta \odot \beta\|_2^2$, and since $V := [v_1 : \cdots : v_p]$ is an orthogonal matrix, one has $\sum_{i > k} |v_{i, \ell}|^2 \le (VV^\top)_{\ell\ell} = 1$.
    Now
    \[
        g_{\ell} = \sum_{i = 1}^k \frac{\|\beta_{\cC_i}\|_2^2}{|\cC_i|}\mathbf{1}_{\cC_i}(\ell).
    \]
    If $t$ sufficiently large so that
    \begin{equation}\label{eq:assm-Xi}
        \Xi(\beta; t, G) \le \frac{1}{2}\min_{i \in A(\beta)}\frac{\|\beta_{\cC_i}\|_2^2}{|\cC_i|} = \frac{1}{2} \min_{\ell \in I(\beta)} g_{\ell},
    \end{equation}
    then for any $\ell \in I(\beta)$,
    \begin{align*}
        \theta_\ell &= \alpha_\ell |g_\ell + \xi_\ell| + (1 - \alpha_\ell) |g_\ell| \\
        &\ge \alpha_\ell (|g_\ell| - \xi_\ell|) + (1 - \alpha_\ell) |g_\ell| \\
        &\ge |g_\ell| - |\xi_\ell| \\
        &\ge g_{\ell} / 2.
    \end{align*}
    It follows that
    \begin{equation}\label{eq:comp-to-glasso-bd-intermediate}
        |\Lambda_t(\beta) - \Lambda_{\infty}(\beta)| \le \Xi(\beta; t, G) \sum_{\ell \in I(\beta)} \frac{1}{\sqrt{2 g_\ell}} + (p - |I(\beta)|) \sqrt{\Xi(\beta; t, G)}.
    \end{equation}
    Now note that
    \begin{align*}
        \sum_{\ell \in I(\beta)} \frac{1}{\sqrt{g_\ell}} &= \sum_{\ell \in I(\beta)} \sum_{i \in A(\beta)} \frac{\sqrt{|\cC_i|}}{\|\beta_{\cC_i}\|_2}\mathbf{1}_{\cC_i}(\ell) \\
                                                         &= \sum_{i \in A(\beta)} \frac{|\cC_i|^{3/2}}{\|\beta_{\cC_i}\|_2} \\
        &\le \sum_{i \in A(\beta)} \frac{|\cC_i|}{\sqrt{2 \Xi(\beta; t, G)}} \quad \text{(using \eqref{eq:assm-Xi})}\\
                                                                                          &= \frac{|I(\beta)|}{\sqrt{2 \Xi(\beta; t, G)}}.
    \end{align*}
    Combining this with \eqref{eq:comp-to-glasso-bd-intermediate}, we get 
    \[
        |\Lambda_t(\beta) - \Lambda_{\infty}(\beta)| \le (p - |I(\beta)|/2) \sqrt{\Xi(\beta; t, G)} \le p \sqrt{\Xi(\beta; t, G)}. 
    \]
    This completes the proof.
\end{proof}

\begin{proof}[Proof of Lemma~\ref{lem:approx-min}]
    For our choice of $\lambda$, any group Lasso solution $\hat{\beta}_{\infty, \lambda}$ satisfies (see Eq. (3.9) of \cite{lounici2011oracle})
    \[
        \frac{1}{n}\|X(\hat{\beta}_{\infty, \lambda} - \beta^*)\|_2^2 \le 8 \lambda \sqrt{T_{\max}} \|\beta^*\|_{2, 1}.
    \]
    Thus under our assumptions, a group Lasso solution $\hat{\beta}_{\infty, \lambda}$ will lie inside $B[\beta^*; \epsilon]$. Now we have
    \begin{align*}
        F_{\infty, \lambda}(\hat{\beta}_{t, \lambda}) &= F_{t, \lambda}(\hat{\beta}_{t, \lambda}) + \Lambda_\infty(\hat{\beta}_{t, \lambda}) - \Lambda_t(\hat{\beta}_{\lambda}) \\
                                                      &\le F_{t, \lambda}(\hat{\beta}_{\infty, \lambda}) + C_{\beta^*, \epsilon_0} p 
                                                      e^{-t\lambda_g/2} \\
                                                      &= F_{\infty, \lambda}(\hat{\beta}_{\infty, \lambda}) + \Lambda_t(\hat{\beta}_{\infty, \lambda}) - \Lambda_\infty(\hat{\beta}_{\infty, \lambda}) + C_{\beta^*, \epsilon_0} p 
                                                      e^{-t\lambda_g/2} \\
                                                      &\le F_{\infty, \lambda}(\hat{\beta}_{\infty, \lambda}) + 2 C_{\beta^*, \epsilon_0} p 
                                                      e^{-t\lambda_g/2}.
    \end{align*}
    This completes the proof.
\end{proof}
\begin{proof}[Proof of Lemma~\ref{lem:approx-oracle}]
    By Lemma~\ref{lem:approx-min}, for any $\beta$,
    \begin{equation}\label{eq:approx-basic}
        F_{\infty, \lambda}(\hat{\beta}_{t, \lambda}) \le F_{\infty, \lambda}(\beta) + E.
    \end{equation}
    This may be thought of as an approximate ``basic inequality'' for the estimator $\hat{\beta}_{t, \lambda}$. Now we use the arguments used in the proof of the sparsity oracle inequality for group Lasso in \cite{lounici2011oracle}. Eq. \eqref{eq:approx-basic} gives
    \[
        \frac{1}{2n}\|X \hat{\beta}_{t, \lambda} - y\|_2^2 +  \lambda \sum_{j} \sqrt{T_j} \|\hat{\beta}_{t, \lambda}^j\|_2 \le \frac{1}{2n}\|X \beta - y\|_2^2 +  \lambda \sum_{j} \sqrt{T_j} \|\beta^j\|_2 + E,
    \]
    where for a vector $\beta \in \R^p$, we use the shorthand $\beta^j \equiv \beta_{\cC_j}$ to reduce the notational overload in subscripts. Putting $y = X\beta^* + \varepsilon$ and rearranging the resulting expression, we get
    \begin{align}\label{eq:basic-ineq-rearranged}\nonumber
        \frac{1}{2n}\|X(\hat{\beta}_{t, \lambda} &- \beta^*)\|_2^2\\ &\le \frac{1}{2n}\|X(\beta - \beta^*)\|_2^2 + \frac{1}{n} \varepsilon^\top X(\hat{\beta}_{t, \lambda} - \beta) + \lambda \sum_{j} \sqrt{T_j} (\|\beta^j\|_2 - \|\hat{\beta}_{t, \lambda}^j\|_2) + E.
    \end{align}
    By Cauchy-Schwarz, we have
    \[
        \varepsilon^\top X (\hat{\beta}_{t, \lambda} - \beta) \le \sum_j \|\varepsilon^\top X_j\|_2 \|\hat{\beta}_{t, \lambda}^j - \beta^j\|_2.
    \]
    Consider the events $\mathcal{A}_j = \{n^{-1}\|\varepsilon^\top X_j\|_2 \le \lambda \sqrt{T_j}/2\}$. Since $n^{-1}X_j^\top\varepsilon \sim N\bigg(0, \sigma^2 \frac{X_j^\top X_j}{n}\bigg)$, we have using Lemma B.1 of \cite{lounici2011oracle} that
    \[
        \P(\mathcal{A}_j^c) \le 2\eta
    \]
    provided
    \[
        \lambda \ge \frac{2\sigma}{\sqrt{n}} \sqrt{\frac{1}{T_j}(\tr(\Psi_j) + 2 \|\Psi_j\|_{\op} (2 \log \eta^{-1} + \sqrt{T_j \log \eta^{-1}}))}.
    \]
    A simpler sufficient condition for this is
    \[
        \lambda \ge \frac{2\sigma}{\sqrt{n}} \sqrt{\|\Psi_j\|_{\op}}\bigg(1 +  \sqrt{\frac{4\log \eta^{-1}}{T_j}}\bigg).
    \]
    Let
    \[
        \Lambda(X; \eta) = \max_{1 \le j \le k} \sqrt{\|\Psi_j\|_{\op}}\bigg(1 +  \sqrt{\frac{4\log \eta^{-1}}{T_j}}\bigg).
    \]
    Thus, if $\lambda \ge \frac{2\sigma}{\sqrt{n}}\Lambda(X; \eta)$, then with probability at least $1 - 2k\eta$, we have
    \begin{equation}\label{eq:epsXbeta}
        \frac{1}{n}\varepsilon^\top X (\hat{\beta}_{t, \lambda} - \beta) \le \frac{\lambda}{2} \sum_j \sqrt{T_j} \|\hat{\beta}_{t, \lambda}^j - \beta^j\|_2
    \end{equation}
    Combining \eqref{eq:epsXbeta} with \eqref{eq:basic-ineq-rearranged} we get that with probability at least $1 - 2k\eta$, we have
    \begin{align*}
        \frac{1}{2n}\|X(\hat{\beta}_{t, \lambda} - \beta^*)\|_2^2 &+ \frac{\lambda}{2} \sum_j \sqrt{T_j} \|\hat{\beta}_{t, \lambda}^j - \beta^j\|_2 \\
                                                                  &\le \frac{1}{2n} \|X(\beta - \beta^*)\|_2^2 + \lambda \sum_{j} \sqrt{T_j} (\|\beta^j\|_2 - \|\hat{\beta}_{t, \lambda}^j\|_2 + \|\hat{\beta}_{t, \lambda}^j - \beta^j\|_2) + E \\ 
                                                                  &\le \frac{1}{2n} \|X(\beta - \beta^*)\|_2^2 + 2 \lambda \sum_{j \in A(\beta)} \sqrt{T_j} \min \{\|\beta^j\|_2, \|\hat{\beta}_{t, \lambda}^j - \beta^j\|_2\} + E,
    \end{align*}
    where in the last line we have used the fact that
    \[
        \|\beta^j\|_2 - \|\hat{\beta}_{t, \lambda}^j\|_2 + \|\hat{\beta}_{t, \lambda}^j - \beta^j\|_2 \le \begin{cases}
            0 & \text{ if } j \notin A(\beta), \\
            2 \min\{\|\beta^j\|_2, \|\hat{\beta}_{t, \lambda}^j - \beta^j\|_2\} & \text{ if } j \in A(\beta).
        \end{cases}
    \]
    This is the desired approximate sparsity oracle inequality \eqref{eq:approx-sparsity-oracle}.
\end{proof}


\section{Conclusion}
In this work, we contribute an approach to learning under a group structure on explanatory variables that does not require prior information on the group identities. Our paradigm is motivated by the Laplacian geometry of an underlying network with a commensurate community structure, and proceeds by directly incorporating this into the penalty. In a more general setup, when an underlying graph may not be explicit in the problem description, we demonstrate a procedure to construct such a network based on the available data. Notably, we dispense with the elaborate pre-processing step involving clustering of the variables, spectral or otherwise, which can be computationally resource-intensive. Our paradigm is underpinned by rigorous theorems that guarantee effective performance and provide bounds on its sample complexity. In particular, we demonstrate that in a very wide range of settings, we need to run the diffusion for a time that is only logarithmic in the problem dimensions. We investigate in detail the interplay of our approach with key statistical physics paradigms such as the GFF and the SBM. We validate our approach by successful application to real-world data from diverse fields including computer science, genetics, climatology and economics.

Our approach opens the avenue to applications of similar dynamical techniques to classical statistical and data analytical problems, that are normally defined as static problems. The inherently local nature of the heat flow and related diffusion algorithms enables us to resolve the relevant constrained optimization problems while being oblivious to the global geometry of the graph (such as a complete understanding of the clustering structure of the variables). In addition to economies of computational resource, such locality is of significance in the context of questions of privacy in data analytical methodologies, a problem that is gaining increasing salience in today's hyper-networked world. A redeeming feature of our approach is that we do not require prior knowledge of the number of groups or clusters, which is a significant departure from usual network-based techniques that makes it more widely applicable to real-world scenarios where such detailed information may not be available.

On a related note, it would be of interest to enhance our approach to obtain similarly local algorithms that address additional structural features of the explanatory variables, such as smoothness or intra-group sparsity. Yet another intriguing direction would be to explore the interface of our approach and diffusion-mapping based techniques that have been effective for dimension reduction problems, and exploit their interplay to achieve further economy of scale and computational resources. In general, the interplay between the geometric structure provided by the Laplacian, the stochastic structure accorded by models such as the GFF and SBM and the inherent clustering structure of real world datasets raises the possibility of a rich mathematical theory and a suite of associated techniques to evolve.

\section*{Acknowledgements}
SG was supported in part by the Singapore Ministry of Education grants R-146-000-250-133,
R-146-000-312-114, A-8002014-00-00 and MOE-T2EP20121-0013. SSM was partially supported by an INSPIRE research grant (DST/INSPIRE/04/2018/002193) from the Department of Science and Technology, Government of India; a Start-Up Grant and the CPDA from Indian Statistical Institute; and a Prime Minister Early Career Research Grant (ANRF/ECRG/2024/006704/PMS) from the Anusandhan National Research Foundation, Government of India. The authors thank Snigdhansu Chatterjee for pointers to the NCEP/NCAR reanalysis data set. The authors would like to thank Peter B{\"uh}lmann and Robert Tibshirani for illuminating discussions.

\bibliographystyle{apalike}
\bibliography{glasso}

\appendix
\renewcommand{\theequation}{S.\arabic{equation}}
\setcounter{equation}{0}
\section{Generalities}

\subsection{The generator of the heat flow}

In this section we demonstrate that the generator of the heat flow in Algorithm 1 is indeed the graph Laplacian. To this end, we denote $X_t$ to be the location of this continuous time Markov Chain at time $t$. Let $f$ be a test function on the graph $G$. 

Let $v \in G$ be a vertex with degree $\deg(v)$. For $\del>0$ small, we proceed to compute $\E[f(X_{t+\del}) | X_t = v]$. The probability of there being multiple jumps of the Markov Chain in time $\del$ is $O(\del^2)$, and therefore, to understand the above expectation to the first order in $\del$, we focus on the situation where there is at most one jump in the time interval $(t,t+\del)$.

The next jump in the Markov Chain occurs when the exponential clock along any of the edges of $G$ incident on $v$ rings. Since these clocks are i.i.d. with parameter 1 each, we  the timing of the next jump is the minimum of $\deg(v)$ many i.i.d. Exponential (1) random variables. The latter random variable is easily verified to be an Exponential ($\deg(v)$) random variable, whose mean is $1/\deg(v)$. 
We have, $\P[\text{No jump in } (t,t+\del)] = \exp(-\del \deg(v))$.  If there is a jump in the time interval $(t,t+\del)$, the Markov Chain moves to a neighbouring vertex of $v$ chosen uniformly at random, each with probability $1/\deg(v)$. 

Therefore, we may write 
\begin{equation}
    \E[f(X_{t+\del}) | X_t = v] = \exp(-\del \deg(v)) \cdot f(v) + (1 - \exp(-\del \deg(v))) \cdot \frac{1}{\deg(v)} \cdot \l( \sum_{u \sim v} f(u) \r).
\end{equation}

This implies that 
\begin{align*} 
    \E[f(X_{t+\del}) &| X_t = v] - f(v) \\
    = & \l(1 - \exp(-\del \deg(v))\r) \cdot \l[\frac{1}{\deg(v)} \cdot \l( \sum_{u \sim v} f(u) \r)   - f(v)\r] \\
    = & \l(1 - \exp(-\del \deg(v))\r) \cdot \frac{1}{\deg(v)} \cdot \l[ \sum_{u \sim v} (f(u) - f(v)) \r]. 
\end{align*}

Therefore
\begin{align*} 
    [\cG f](v) = & \lim_{\del \to 0} \frac{1}{\del} \cdot \l(\E[f(X_{t+\del}) | X_t = v] - f(v) \r) \\
    = & \lim_{\del \to 0} \frac{1}{\del} \cdot \l(1 - \exp(-\del \deg(v))\r) \cdot \frac{1}{\deg(v)} \cdot \l[ \sum_{u \sim v} (f(u) - f(v)) \r] \\ 
    = & \; \deg(v) \cdot \frac{1}{\deg(v)} \cdot \l[ \sum_{u \sim v} (f(u) - f(v)) \r] \\
    = &  \l[ \sum_{u \sim v} (f(u) - f(v)) \r] \\
    = & \; [L f](v),
\end{align*}
where $L$ is the standard (unnormalised) graph Laplacian of $G$.

This completes the proof that the the standard (unnormalised) graph Laplacian $L$ of $G$ is the generator of the continuous time Markov Chain in Algorithm 1.

\subsection{Continuous-time vs. discrete-time random walks}\label{sec:cont-vs-discrete}
In our formulation of the heat-flow penalty, we have crucially used the fact that the unnormalised Laplacian matrix $L$ is the infinitesimal generator of the continuous-time random walk on a finite graph $G$ and hence, if the graph is connected, the uniform distribution is the unique stationary distribution. One might also use the discrete-time random walk on $G$. However, for that one needs some adjustments and the resulting penalty is slightly clumsier.

Let $P = D^{-1}A$ denote the transition probability matrix for the discrete-time random walk $(\tilde{Z}_{N})_{N \ge 0}$ on $G$. Recalling that the discrete-time random walk on a connected graph is reversible with stationary distribution
\[
    \pi_i \propto \deg(i),
\]
and since for any $x \in \R^V$ and $i \in V$, one has the identity
\[
    (P^Nx)_i = \E[x_{\tilde{Z}_N} \mid \tilde{Z}_0 = i], 
\]
one may show, akin to the continuous-time setting, that for sufficiently large $N$, 
\[
    \tilde{\Lambda}_N(\beta) := \langle \Phi(P^N(\beta \odot \beta \odot d^{[-1]})), \bone\rangle \approx \mathrm{GL}(\beta),
\]
where $d^{[-1]} := (\deg(i)^{-1} \ind(\deg(i) \ne 0))_{i = 1}^p$. Similarly, one may check using reversibility that
\[
    \nabla_{\beta} \tilde{\Lambda}_N(\beta) = P^N(\zeta \odot d^{[-1]}) \odot \beta,
\]
which is the analogue of \eqref{eq:Lambda_gradient}. The continuous-time version clearly offers a much cleaner formulation.


 
\subsection{Completely disconnected vs. rarely inter-connected groups} \label{sec:discon-vs-sparscon}

While our theoretical considerations largely focus on the setting where  inter-group connections are absent, for practical purposes, our paradigm is applicable to settings where connections across groups are \textit{rare} but not completely absent; such a scenario being treated as an approximation or a minor deformation of complete disconnection. In the latter setting, it is conceivable that the 0 eigenvalue in the graph Laplacian spectrum has multiplicity only 1; on the other hand there would be a part of the Laplacian spectrum that is very close to 0 but not exactly equal to 0 (for brevity, we will denote it by $\slow$; the full Laplacian spectrum being denoted by $\spec$). Intuitively, this is reflective of the fact that the graph has a group structure that is not fully disconnected, but only rarely connected. If we modified the graph to remove these rare connections across components, these low-lying spectrum of the Laplacian would collapse to 0, and we would be back to the setting of complete disconnection between groups. 

In such a scenario, if the inter-group connections are rare compared to the intra-group connections, we would still expect the rest of the Laplacian spectrum (i.e., $\spec \setminus \slow$) to be well-separated from the above low-lying eigenvalues. As such, our substitute for $\la_g$ would be \newline $\min \l\{ \la : \la \in \spec \setminus \slow       \r\}$. If we denote $\la_\low$  to be $\max \l\{ \la : \la \in \slow   \r\}$, we are operating in the regime where $\la_g \gg \la_\low$. 

In view of these considerations, in the setting of rare but non-zero connections across groups, our heat flow time $\tflow$ needs to be such that $ \tflow \cdot \la_\low $ is small, but $ \tflow \cdot \la_g $ is large. This necessitates a choice of $\tflow$ such that $\frac{1}{\la_g} \ll \tflow \ll \frac{1}{\la_\low}$. Since $\la_\low \ll \la_g$, this enables us to make appropriate choice of the heat flow duration that extend our approach to the setting of rarely connected groups.

\section{Analysis of sample complexity and prediction guarantees for random designs}    \label{sec:analysis_models}

\subsection{Prediction guarantees and sample complexity bounds}  \label{sec:recovery-app}
In this section, we will demonstrate quantitative  guarantees for the prediction error and sample complexity in our heat flow based approach for key models of random designs with a group structure. These include, in particular, GFFs on typical clustered networks and Gaussian designs based on SBMs. To this end, we appeal to Theorem \ref{thm:prediction-consistency}. As we shall see below, both settings satisfy the RE($s$) property, so via Theorem \ref{thm:prediction-consistency} and Proposition~\ref{prop:res-eig}, we will obtain concrete prediction error guarantees as well as bounds on the sample complexity as soon as we can bound the quantities $\la, \rho(\S)$ and $\kappa_\S(s)$.

We first turn our attention to the quantity $\la$. To this end, we invoke \eqref{eq:lambda_bounds} and Lemma \ref{lem:approx-min}.  In the present article, we will content ourselves with a polynomial decay of probability, which implies that the quantity $\eta=O(n^{-\a})$ for some $\a$. Since the maximal group size $|\calC_{\max}| \ge 1$, we have $\l( 1+ \sqrt{4 \log \eta^{-1}}{|\calC_{\max}|}| \r)=O(\sqrt{\log n})$.  On the other hand, we have the bound $\max_j \|\Psi_j\|_{\mathrm op}=\max_j \|\frac{1}{n}X_j^TX_j\|_{\mathrm op}=O_P(\s_{\max}(\S))$, where $\s_{\max}(\S)$ is the maximal singular value of the population covariance matrix $\S$.    \eqref{eq:lambda_bounds} therefore implies that $\L(X;\eta)=O_P(\s_{\max}(\S)\sqrt{\log n})$. As a result, Lemma \ref{lem:approx-min} suggests that we consider $\la \gtrsim  \s_{\max}(\S) \sqrt{\frac{\log n}{n}}$.

We next observe that if $\s_{\max}(\S)$ and $\s_{\min}(\S)$ are respectively the maximum and minimum singular values of the population covariance matrix $\S$, then $\rho(\S) = \max_i \S_{ii} \le \s_{\max}(\S)$ and $\kappa_\S(s) \ge \s_{\min}(\S)$, which are direct consequences of the definitions of the quantities in question. In view of Theorem \ref{thm:prediction-consistency} and, in particular \eqref{eq:PC_fast}, this leads to a somewhat simplified prediction guarantee of 
\begin{equation} \label{eq:PG-app}
    \frac{1}{n}\|X(\hat{\beta}_{t, \lambda} - \beta^*)\|_2^2 = O_P\l(s\; |\calC_{\max}| \cdot \frac{\s_{\max}(\S)}{\s_{\min}(\S)^2} \cdot \frac{\log n}{n} + p e^{-t\lambda_g/2}  \r) . 
\end{equation} 
Our goal here is to understand the order of the time $\tflow$ and the step count $\Ns$ (in terms of the other parameters of the problem) till which we need to run our heat flow based algorithm in order to achieve a desired accuracy.  We will  bifurcate our analysis into two related sections.

\subsubsection{Bounds on \texorpdfstring{$\tflow$}{} and \texorpdfstring{$\Ns$}{} for given \texorpdfstring{$n, p$}{}}

First, given $n,p$, we will demonstrate the order of $\tflow$ at which the approximation error due to our heat flow based approach (roughly, the second term in \eqref{eq:PG-app}) becomes comparable to  the contribution to the prediction error bound from the classical group lasso methods that assume complete knowledge of the group structure (roughly, the first term in \eqref{eq:PG-app}). Heuristically, this indicates the time we need to run the heat flow in order to be comparable to the classical group lasso (but without requiring complete knowledge of the groups, unlike the classical setting).  Equating the two terms in \eqref{eq:PG-app}, we deduce that it suffices to take
\begin{equation} \label{eq:time_10-app}
    \tflow ~\gtrsim \frac{1}{\la_g} \log p + \frac{1}{\la_g} \log \l( \frac{1}{s\; |\calC_{\max}|} \cdot  \frac{ \s_{\min}(\S)^2} {s \s_{\max}(\S)} \cdot \frac{n}{\log n}\r). 
\end{equation}
In most settings of interest, we have bounds of the form $ p^{-a} \lesssim \s_{\min}(\S) \le  \s_{\max}(\S) \lesssim p^b$ for some $a,b \ge 0$. In particular, this holds for the GFF and block model strucutred covariates that we discuss in the present work. Further, we have the trivial  bounds entailing $s,|\calC_{\max}| \in [1,p]$. Combining these observations with \eqref{eq:time_10-app}, we deduce that for such models, we have the much simpler prescription 
\begin{equation} \label{eq:time_11}
    \tflow ~\gtrsim \frac{1}{\la_g} \max \{ \log p , \log n \}. 
\end{equation}

As discussed in the main text, this implies it suffices to have
\[
    \Ns=O\l(d_{\max} \cdot \frac{1}{\la_g} \cdot \max \{ \log p , \log n \}\r).
\]

For a definite quantitative ballpark for $\tflow$ and $\Ns$, we focus on the setting of typical clustered networks (c.f. Sec. \ref{sec:graph_struct}.). In such a setting, we may deduce that $d_{\max}=\Theta_P(p)$ whereas $\la_g=\Theta_P(p)$; for details we refer the reader to Sec. \ref{sec:graph_struct}.

This implies that we can further simplify to the prescriptions with high probability
\begin{equation}
    \tflow ~\gtrsim \frac{1}{p} \cdot \max \{ \log p , \log n \}
\end{equation}
and 
\begin{equation}
    \Ns =O_P\l(\max \{ \log p , \log n \}\r).
\end{equation}

\subsubsection{Bounds on \texorpdfstring{$n, \tflow, \Ns$}{} for given prediction guarantee \texorpdfstring{$\eps$}{}}

Herein, we fix a prediction error guarantee $\eps$, and make explicit prescriptions for the order of $n, \tflow$ and $\Ns$ that will allow us to obtain a prediction error of order $O(\eps)$. To this end, we posit that the two terms on the right hand side of  \eqref{eq:PG-app} are separately $O(\eps^2)$ (since the left hand side is the squared prediction error). For $\tflow$, this entails that $p e^{-\la_g \tflow / 2} \lesssim \eps^2$, which translates into 
\begin{equation} \label{eq:time_2}
    \tflow \gtrsim \frac{1}{\la_g}\l( \log p + \log \frac{1}{\eps} \r). 
\end{equation}

In the setting for typical clustered networks (c.f. Sec. \ref{sec:graph_struct}), we may deduce that $d_{\max}=\Theta_P(p)$ whereas $\la_g=\Theta_P(p)$ , implying  that we can further simplify  to the following bounds that hold with high probability:
\begin{equation} 
    \tflow ~\gtrsim \frac{1}{p} \cdot \max \l\{ \log p , \log \frac{1}{\eps} \r\}; \quad \Ns =O_P\l(\max \l\{ \log p , \log \frac{1}{\eps} \r\}\r).
\end{equation}

For prescribing $n$ under a target prediction error $\eps$, we need to satisfy two conditions: 

(a) the first term on the right  in \eqref{eq:PG-app} is $O(\eps^2)$, which leads to 
\[    \frac{n}{\log n} \gtrsim ~\frac{1}{\eps^2}  \l( s\; |\calC_{\max}| \cdot  \frac{ \s_{\max}(\S)  }{ \s_{\min}(\S)^2}\r).   \]

(b) As per Proposition~\ref{prop:res-eig} and the bound $\kappa_{\Sigma}(s) \ge \s_{\min}(\S)$, we have 
\[n \gtrsim s \; |\calC_{\max}| \cdot \frac{(\rho(\Sigma))^2}{(\s_{\min}(\S))^2} \cdot \log p. \]

Combining the last two bounds, we obtain the unified bound 
\begin{equation} \label{eq:n_eps-app}
    \frac{n}{\log n} \gtrsim\max \l\{  \l( \frac{1}{\eps^2} \cdot s |\calC_{\max}| \cdot \frac{ \s_{\max}(\S)  }{ \s_{\min}(\S)^2}\r), \l( s \; |\calC_{\max}| \cdot \frac{(\rho(\Sigma))^2}{(\s_{\min}(\S))^2} \cdot \log p \r) \r\}.
\end{equation}

\subsection{On the structure of typical clustered networks} \label{sec:graph_struct}

In this section, we explore the structure of the \textit{typical clustered network}  $G$ on $p$ vertices and with $k$ clusters, where $k=O(1)$, and the size of each cluster is $\Theta(p)$. We model a \textit{typical clustered network} with these parameters as follows. We posit that the clusters $\{\calC_i\}_{i=1}^k$ are fully disconnected across clusters; thus the graph $G$ has exactly $k$ connected components given by the $\calC_i$-s. Each component $\calC_i$ is modelled as a \textit{dense} random graph with $|\calC_i|$ vertices and edge connection probability $\xi_i \in (0,1)$. For the sake of definiteness, we allow self-loops in our model; this would not exert any major influence on the large scale properties of the graph. Since the number of components $k$ is $O(1)$, we may take the $\xi_i$-s to be bounded away from 0 and 1, as the parameters $n$ and $p$ grow.

\subsubsection{Maximum and minimum degrees of typical clustered networks} \label{sec:TCN-deg}

We first demonstrate that the maximum and minimum degrees of $G$, denoted resp. $d_{\max}$ and $d_{\min}$, are both $\Theta_P(p)$. To this end, we observe that if $d_{\max,i}$ and $d_{\min,i}$ are resp. the maximum and minimum degrees of $\calC_i$, then $d_{\max}=\max_{1 \le i \le k} d_{\max,i}$ and $d_{\min}=\min_{1 \le i \le k} d_{\min,i}$.

For any particular vertex $v \in \calC_i$, its degree $\deg(v)$ is distributed as Binomial($T_i,\xi_i,$), where we recall the notation that $T_i=|\calC_i|$. Clearly, $\E[\deg(v)]=T_i \xi_i$. By a well-known  large deviation estimate (c.f. \cite{dembo_zeitouni_2010}), for any $\del>0$, we have 
\begin{equation}
    \P \l[ (1-\del) T_i \xi_i  \le  \deg(v)  \le (1+\del) T_i \xi_i  \r] \ge 1 - \exp(-C(\del,\xi_i)T_i),
\end{equation}
where $0<C(\del,\xi_i)<\infty$ is a quantity that depends only on $\del$ and $\xi_i$. By a union bound, this implies that 
\begin{equation}
    \P \l[ \forall \; v \in \calC_i, \; \text{it holds that} \; (1-\del) T_i \xi_i  \le  \deg(v)  \le (1+\del) T_i \xi_i  \r] \ge 1 - T_i \cdot \exp(-C(\del,\xi_i)T_i).
\end{equation}
But for all $v \in \calC_i$, it holds that
\[
    (1-\del) T_i \xi_i  \le  \deg(v)  \le (1+\del) T_i \xi_i  \equiv (1-\del) T_i \xi_i  \le  d_{\min,i} \le d_{\max_i}  \le (1+\del) T_i \xi_i. 
\]
Since $d_{\max}=\max_{1 \le i \le k} d_{\max,i}$ and $d_{\min}=\min_{1 \le i \le k} d_{\min,i}$, by a further union bound we may deduce that 
\begin{equation}
    \P \l[ (1-\del) \min_{1\le i \le k} (T_i \xi_i)  \le  d_{\min,i} \le d_{\max_i}  \le (1+\del) \max_{1\le i \le k}(T_i \xi_i)  \r] \ge 1 -  \l(\sum_{i=1}^k T_i \cdot \exp(-C(\del,\xi_i)T_i) \r).
\end{equation}

Since in our model of typical clustered networks, $k=O(1)$ while each $T_i=\Theta(p)$ and $\xi_i$ are bounded away from 0 and 1, we may deduce that for any $\del>0$ we have constants $C_1,C_2$ and $C\l(\del,\{\xi_i\}_{i=1}^k\r)$
\begin{equation}
    \P \l[ (1-\del) C_ 1 \cdot p \le  d_{\min} \le d_{\max}  \le (1+\del) C_2 \cdot p \r] \ge 1 -  \exp\l(-C\l(\del,\{\xi_i\}_{i=1}^k\r)p\r).
\end{equation}

This demonstrates that, for a typical clustered network on $p$ nodes, we have $d_{\max}=\Theta_P(p)$ as well as $d_{\min}=\Theta_P(p)$.

\subsubsection{Spectral gap for typical clustered networks} \label{sec:TCN-spec_gap}

In this section, we investigate the order of the spectral gap for typical clustered networks, as defined above. To this end, we invoke results of Chung-Lu-Vu on spectra of random graphs with given expected degree \citep{ChungPNAS,chung2004spectra}. To state their results, we define the random graph model with a given expected degree sequence as follows. 

For a graph on $p$ nodes and a given sequence of non-negative reals $\mathbf{w}=(w_1,\ldots,w_p)$ satisfying $\max_i w_i^2 \le \sum_{k=1}^p w_k$, we define the random graph with the expected degree sequence \citep{ChungPNAS,chung2004spectra} by connecting the vertices labelled $i$ and $j$ with an edge with probability $w_iw_j \rho$, where $\rho=(\sum_{k=1}^p w_k)^{-1}$. A classical random graph $G(p,\xi)$ with $p$ vertices and edge connection probabilities $\xi$ is obtained in the above model by choosing $w_i=p\xi$. 

Two quantities of relevance in the results of \cite{ChungPNAS,chung2004spectra} are the expected average degree $\md=\frac{1}{n} \sum_{k=1}^p w_k$ and the second order average degree  $\snd = \l(\sum_{k=1}^p w_k^2\r) \big/ \l(\sum_{k=1}^p w_k\r)$; we also set $w_{\max}:=\max_k w_k$ and $w_{\min}:=\min_k w_k$. For the $G(p,\xi)$ model considered above, we have $\snd=\md=w_{\max}=w_{\min}=p\xi$. 

For the above random graph model, if the normalised Laplacian $L^* = D^{-1/2} L D^{-1/2}$ (where $L$ is the usual, unnormalised graph Laplacian and $D$ is the diagonal matrix of degrees) has spectrum $0 = \la_0(L^*) \le \la_0(L^*) \le \ldots \le \la_{p-1}(L^*)$, then \cite{ChungPNAS} demonstrates that with high probability we have the bound $\max_{i \ne 0}|1 - \la_i(L^*)| \lesssim \frac{1}{\sqrt{\md}} + o(\frac{\log^3 p}{w_{\min}})$. In our setting of interest, namely the $G(p,\xi)$ random graph, this implies that with high probability we have $\la_1(L^*) \gtrsim 1 - \frac{1}{\sqrt{p \xi}} (1+o(1))$.

We now observe that $\la_1(L)=\la_1(D^{1/2}L^* D^{1/2}) \ge d_{\min}\la_1(L^*)$. But we have already demonstrated that for a $G(p,\xi)$ random graph, we have $d_{\min}=\Theta_P(p)$, which when combined with the analysis above yields $\la_1(L) \gtrsim p$ with high probability for such graphs.

It remains to note that the ground state (equiv., the lowest non-zero eigenvalue of the unnormalised graph Laplacian) $\la_g$ for the \textit{typical clustered network} is given by $\la_g = \min_{1 \le i \le k} \la_1(L_{\calC_i})$, where $L_{\calC_i}$ is the unnormalised graph Laplacian corresponding to the component subgraph $\calC_i$. But in a typical clustered network, each $\calC_i$ is a $G(p,\xi_i)$ random graph for some $\xi_i \in (0,1)$. Since the number of components $k$ is $O(1)$, this leads to the bound $\la_g \gtrsim p$ with high probability.

Finally, we observe via \eqref{eq:sparse_eig_bound} that $\la_g=\la_1(L)\le \la_{\max}(L) \le 2 d_{\max}= \Theta_P(p)$. This, in particular, implies that for typical clustered networks we have $\la_g = \Theta_P(p)$, as desired.  

\subsection{Sample complexity and prediction guarantees for covariate structures based on GFF} \label{sec:recovery_GFF}

\subsubsection{The order of the mass parameter \texorpdfstring{$\t$}{}}
A few words are in order regarding the size of the mass parameter (or convexity parameter) $\t$.
Our motivation behind introducing the parameter $\t$ is to create sufficient convexity to overcome the singular nature of $L$, without changing the essential orders of magnitude associated with the model. 
For the upper bound on the spectrum of the covariance matrix in the GFF case, we  will proceed via the identity $\s_{\max}(\S)=(\s_{\min}(L)+\t)^{-1}$. Since $L$ is singular,  $\s_{\min}(L)=0$; thus $\s_{\max}(\S)=\t^{-1}$.  In view of this, if $\s_{\text{low}}(L)$ is the smallest non-zero eigenvalue of $L$, we will select $\t$ to be simply equal to $\s_{\mathrm{low}}(L)$, thereby adding a minimal amount of convexity without making essential changes to its large scale behaviour. For the spectral lower bound on the covariance structure of GFFs, we clearly have $\s_{\min}(\S)=(\s_{\max}(L)+\t)^{-1}$.

\subsubsection{Spectral bounds}

The study of extremal eigenvalues of graph Laplacians has a long history in spectral graph theory; for a comprehensive account we refer the reader to the surveys \cite{spielman2007spectral,spielman2012spectral}. In this work, we content ourselves with the following general bounds, which are essentially versions of results known in the literature, or follow via simple considerations therefrom. 

\begin{lemma} \label{lem:graph_spectra}
    For $G=(V,E)$ be a connected graph, we denote by $L$ and $L^*$ the unnormalised and normalised graph Laplacians respectively. Then the following general spectral bounds hold.

    We have the spectral upper bounds
    \begin{equation} \label{eq:dense_eig_bound}
        \s_{\max}(L) \le \frac{\sqrt{1+8|E|}-1}{2}
    \end{equation}
    and
    \begin{equation} \label{eq:sparse_eig_bound}
        \s_{\max}(L) \le 2 d_{\max},
    \end{equation}
    where $d_{\max}$ is the maximum degree of the graph $G$.

    On the other hand, we have the spectral lower bound
    \begin{equation} \label{eq:Cheeger_eig_bound}
        \s_{\mathrm{low}}(L^*) \ge \frac{1}{2}\mathfrak{K}(G)^2,
    \end{equation}
    where $\s_{\mathrm{low}}$ is the smallest non-zero eigenvalue of $L^*$, $\mathfrak{K}(G)$ is the conductance of the graph $G$, defined by \[\mathfrak{K}(G):=\min_{S \subset V, S \ne \phi} \frac{E(S,V \setminus S)}{\min\{\mathrm{Vol}(S), \mathrm{Vol}(V \setminus S)\}}\] with $E(S,V\setminus S)$ being the number of edges between the vertices in $S$ and $V \setminus S$, and $\mathrm{Vol}(A)$ for $A \subseteq V$ being defined as $\mathrm{Vol}(A)=\sum_{v \in A} \deg(v)$. 
\end{lemma}

We observe that in our case, the graph underlying the GFF is not connected. We will, therefore, apply Lemma \ref{lem:graph_spectra} to each connected component (i.e., group) of the vertices. The component-wise bounds can then be combined to obtain \[  \min_i \s_{\min}(\S_{\calC_i})  \le \s_{\min}(\S) \le \s_{\max}(\S) \le \max_i \s_{\max}(\S_{\calC_i}). \] We will invoke Lemma \ref{lem:graph_spectra} for each connected component of the underlying graph.

\subsubsection{Random designs on typical clustered networks}

For the GFF on typical clustered networks, we invoke the inequality \eqref{eq:sparse_eig_bound}. As a result, we have $\s_{\max}(L_{\calC_i}) \lesssim \d_{\max}(\calC_i)$. It has been demonstrated in Section \ref{sec:graph_struct} that the last quantity is, with high probability, $\Theta_P(p)$.
We therefore have, with high probability, the bound
\begin{equation} \label{eq:GFF_dense-1}
    \s_{\min}(\S) \simeq \s_{\max}(L)^{-1} = \Theta(p^{-1}). 
\end{equation}

For analysing the behaviour of $\s_{\max}(\S)$, equivalently that of $\s_{\low}(L)$ (which is the smallest non-zero eigenvalue of $L$), we invoke the analysis in Section \ref{sec:graph_struct} and deduce that $\s_{\low}(L)=\la_g=\Theta(p)$ with high probability. 

We therefore have, with high probability, that
\begin{equation} \label{eq:GFF_dense-2}
    \rho(\S) \le  \s_{\max}(\S) \simeq \s_{\low}(L)^{-1}  = O(p^{-1}). 
\end{equation} 

For typical clustered networks, the group sizes are comparable to each other (i.e., their ratios are uniformly bounded in $p,n$), therefore recalling that $s$ is the number of non-vanishing groups  we roughly have $s |\calC_{\max}| \simeq \| \b^* \|_0$, where $ \| \b^* \|_0$ is the $L^0$ norm (equivalently, the support size) of the true signal $\b^*$.
Applying \eqref{eq:GFF_dense-1} and \eqref{eq:GFF_dense-2} to \eqref{eq:n_eps-app}, we may therefore deduce that 
\[ 
    \l( \frac{n}{\log n} \r)_{\text{GFF}}  \gtrsim \max \l\{ \frac{1}{\eps^2} \cdot \|\b^*\|_0 \cdot  p,~~ \|\b^*\|_0 \cdot \log p \r\}. 
\] 
It may be noted that the $\frac{1}{\eps^2} \cdot \|\b^*\|_0 \cdot  p$ term above comes from the first term in \eqref{eq:n_eps-app}, which, roughly speaking, reflects the error incurred by classical group lasso. Thus, the linear dependence of the sample complexity on $p$ appears to be a fundamental characteristic of the problem for GFF based random designs, and is inherent in both classical group lasso and the present heat flow based methods.

\subsection{Sample complexity and prediction guarantees for covariate structures based on Stochastic Block Models} \label{sec:recovery_block} 

\subsubsection{Spectral bounds}

We discuss herein upper and lower spectral bounds on block matrices. This is encapsulated in the following lemma.
\begin{lemma} \label{lem:block-alg_dec}
    We have 
    \begin{equation} \label{eq:block-alg_dec}
        \S = \l(1 - (a-b) \r) I_V + (a-b) \sum_{i=1}^k \ind_{\calC_i} \ind_{\calC_i}^T + b \ind_V \ind_V^T.
    \end{equation}
    Further, we have the bounds 
    \begin{equation} \label{eq:block-bounds}
        (1-a+b) \le \s_{\min}(\S) \le \s_{\max}(\S) \le (1-a+b) + (a-b) \max_i{|\calC_i|} +  b|V|.
    \end{equation}
\end{lemma}

The standard  choice  for connection probabilities in the stochastic block model in the community detection literature entails that $a=\tilde{a}/|V|$ and $b=\tilde{b}/|V|$, with the ratio $\tilde{a}/\tilde{b}$ large but fixed; this is the setting we will work with in the present paper. We may obtain from Lemma \ref{lem:block-alg_dec} that $\s_{\max}(\S)$ and $\s_{\min}(\S)$ are both $O(1)$. We record this as \[ \rho(\S)=1; ~\s_{\max}(\S)= O(1)  \text{ and} ~\kappa_\S(s) \ge  \s_{\min}(\S) \gtrsim O(1). \]

Combining these bounds with \eqref{eq:n_eps-app}, and working in the setting of typical clustered networks which entails balanced group sizes (so that we can approximate $s |\calC_{\max}| \simeq \|\b^*\|_0$), we  obtain the bound 
\[ 
    \l( \frac{n}{\log n} \r)_{\text{SBM}} \gtrsim  \max \l\{ \frac{1}{\eps^2} \cdot \|\b^*\|_0,~~ \|\b^*\|_0 \cdot \log p \r\}. 
\] 

\subsection{Proofs for spectral bounds}  \label{sec:spectral_proofs}

\subsubsection{Proofs of spectral bounds for GFF}

\begin{proof}[Proof of Lemma \ref{lem:graph_spectra}]
    The bound \eqref{eq:dense_eig_bound} has been established by \cite{Stanley}, which we refer the interested reader to for a detailed proof. 

    The key ingredient in \eqref{eq:sparse_eig_bound} is the basic spectral inequality on the adjacency matrix $A$ of a graph, given by $\s_{\max}(A) \le d_{\max}$.  This follows from the non-negative definiteness of the Laplacian: \[ 0 \preccurlyeq L = D - A \implies A \preccurlyeq D \implies \|A\|_{\mathrm{op}} \le \|D\|_{\mathrm{op}}  \implies \s_{\max}(A) \le \s_{\max}(D) = d_{\max}. \]
    As result, we may write
    \[ \s_{\max}(L) = \|L\|_{\mathrm{op}} = \|D-A\|_{\mathrm{op}} \le \|D\|_{\mathrm{op}} + \|A\|_{\mathrm{op}} \le 2d_{\max}. \]

    Finally, the spectral lower bound \eqref{eq:Cheeger_eig_bound} is one direction of the celebrated Cheeger's inequality on the smallest non-zero eigenvalue of the normalised graph Laplacian \citep{spielman2007spectral,spielman2012spectral}.
\end{proof}

\subsubsection{Proofs of spectral bounds for block models }

\begin{proof}[Proof of Lemma \ref{lem:block-alg_dec}.]
    The expression 
    \eqref{eq:block-alg_dec} follows from a simple algebraic decomposition of the matrix $\S=I_V + A$, which can be verified via a direct computation. 

    In order to obtain the singular value bounds on $\S$, we will make repeated use of the following spectral inequality for non-negative definite (\textit{abbrv.}  n.n.d.) matrices. Suppose 
    $\{B_i\}_{i=1}^{k_1}$ and $\{C_i\}_{i=1}^{k_2}$ are n.n.d. matrices such that for some n.n.d. matrix $M$ we have
    \begin{equation} \label{eq:eig_dom_1}
        \sum_{i=1}^{k_1} B_i  \preccurlyeq M \preccurlyeq \sum_{j=1}^{k_2} C_j,
    \end{equation}
    where $\preccurlyeq$ denotes smaller than or equal to in the n.n.d. order.
    Then we must have
    \begin{equation} \label{eq:eig_dom_2}
        \max_i \s_{\min}(B_i)  \le  \s_{\min}(M) \le \s_{\max}(M) \le \sum_{j=1}^{k_2} \s_{\max}(C_j).
    \end{equation}

    We now apply the \eqref{eq:eig_dom_2} to $M=\S$ with  $k_1=k_2=3;  ~B_{1}=C_{1}=\l(1 - (a-b) \r) I_V; ~B_2=C_2= (a-b) \sum_{i=1}^k \ind_{\calC_i} \ind_{\calC_i}^T $,  and $B_{3}=C_{3}=b \ind_V \ind_V^T$. 

    In order to deal with $\s_{\max}(B_2)$, it remains to observe that for any subset $S \subseteq V$ and any non-negative scalar $c$, we have $\s_{\min}(c \ind_S \ind_S^T)=0$ and $\s_{\max}(c \ind_S \ind_S^T) = c |S|$, and for two such \textit{disjoint} subsets $S_1,S_2 \subset V$, we have 
    \[   \s_{\max}( c \ind_{S_1} \ind_{S_1}^T + c \ind_{S_1} \ind_{S_1}^T)  \le c \max\{ \s_{\max}(\ind_{S_1} \ind_{S_1}^T, \ind_{S_2} \ind_{S_2}^T)  \} . \]
\end{proof}

\end{document}